\documentclass[preprint,12pt]{elsarticle}
\usepackage[switch]{lineno}
\usepackage{cite}
\usepackage{graphicx}
\usepackage{dblfloatfix}  
\usepackage{leftidx}
\usepackage{amsmath,amsfonts,amsthm}
\DeclareMathOperator*{\argmax}{argmax}
\DeclareMathOperator*{\argmin}{argmin}
\usepackage{bm} 
\usepackage{booktabs} 
\usepackage[ruled,vlined]{algorithm2e}
\usepackage{caption}
\usepackage[colorlinks,
            linkcolor=black,
            anchorcolor=black,
            citecolor=black,
            urlcolor = black
            ]{hyperref}
            
\newcommand{\etal}{\textit{et al. }}
\usepackage{amsthm}
\newtheorem{theorem}{Theorem}

\usepackage{tabularx}

\usepackage{subfigure}
\usepackage{caption}
\usepackage[dvipsnames]{xcolor}
\definecolor{dark_blue}{HTML}{0000A0}
\definecolor{dark_red}{HTML}{8b0000}

\newcommand{\RNum}[1]{\uppercase\expandafter{\romannumeral #1\relax}}

\journal{Information Fusion}

\begin{document}
\renewcommand*{\figureautorefname}{Fig.}
\captionsetup[figure]{name={Fig.}}

\begin{frontmatter}

\title{Ensemble diverse hypotheses and knowledge distillation for unsupervised cross-subject adaptation}

\author[ShenzhenKeyLab,GuangdongKeyLab,ubcmech]{Kuangen Zhang}
\ead{kuangen.zhang@alumni.ubc.ca}

\author[ubcmech]{Jiahong Chen}
\ead{jiahong.chen@ieee.org}

\author[ubcee]{Jing Wang}
\ead{jing@ece.ubc.ca} 

\author[ShenzhenKeyLab,GuangdongKeyLab]{Xinxing Chen}
\ead{chenxx@sustech.edu.cn}

\author[ShenzhenKeyLab,GuangdongKeyLab]{\\Yuquan Leng}
\ead{lengyq@sustech.edu.cn}

\author[ubcmech]{Clarence W. de Silva}
\ead{desilva@mech.ubc.ca}

\author[ShenzhenKeyLab,GuangdongKeyLab]{Chenglong 
Fu\corref{mycorrespondingauthor}}
\ead{fucl@sustech.edu.cn}

\address[ShenzhenKeyLab]{Shenzhen Key Laboratory of Biomimetic Robotics and Intelligent Systems, Department of Mechanical and Energy Engineering, Southern University of Science and Technology, Shenzhen, 518055, China}
\address[GuangdongKeyLab]{Guangdong Provincial Key Laboratory of Human-Augmentation and Rehabilitation Robotics in Universities, Southern University of Science and Technology, Shenzhen, 518055, China}
\address[ubcmech]{Department of Mechanical Engineering, The University of British Columbia, Vancouver, BC, Canada}
\address[ubcee]{Department of Electrical and Computer Engineering, The University of British Columbia, Vancouver, BC, Canada}

\cortext[mycorrespondingauthor]{Corresponding author}
\cortext[code]{Code and data: \url{https://github.com/KuangenZhang/EDH}}

\begin{abstract}
\textcolor{black}{
Recognizing human locomotion intent and activities is important for controlling the wearable robots while walking in complex environments. However, human-robot interface signals are usually user-dependent, which causes that the classifier trained on source subjects performs poorly on new subjects. To address this issue, this paper designs the ensemble diverse hypotheses and knowledge distillation (EDHKD) method to realize unsupervised cross-subject adaptation. EDH mitigates the divergence between labeled data of source subjects and unlabeled data of target subjects to accurately classify the locomotion modes of target subjects without labeling data. Compared to previous domain adaptation methods based on the single learner, which may only learn a subset of features from input signals, EDH can learn diverse features by incorporating multiple diverse feature generators and thus increases the accuracy and decreases the variance of classifying target data, but it sacrifices the efficiency. To solve this problem, EDHKD (student) distills the knowledge from the EDH (teacher) to a single network to remain efficient and accurate. The performance of the EDHKD is theoretically proved and experimentally validated on a 2D moon dataset and two public human locomotion datasets. Experimental results show that the EDHKD outperforms all other methods. The EDHKD can classify target data with 96.9\%, 94.4\%, and 97.4\% average accuracy on the above three datasets with a short computing time (1 ms). Compared to a benchmark (BM) method, the EDHKD increases 1.3\% and 7.1\% average accuracy for classifying the locomotion modes of target subjects. The EDHKD also stabilizes the learning curves. Therefore, the EDHKD is significant for increasing the generalization ability and efficiency of the human intent prediction and human activity recognition system, which will improve human-robot interactions.} 
\end{abstract}




\end{frontmatter}

\section{Introduction}
Wearable robots, including powered prostheses \citep{clites_proprioception_2018, azocar_design_2020}, exoskeletons \citep{zhang_human---loop_2017, ding_human---loop_2018}, and supernumerary robotic limbs \citep{hao_supernumerary_2020}, are expected to assist humans to regain mobility, improve energy economy, and enhance load carrying. The potential of wearable robots looks promising, but challenges remain in developing an adaptive and intelligent controller to control a wearable robot. A typical controller of a wearable robot has a hierarchical architecture, including high-level, middle-level, and low-level controllers \citep{tucker_control_2015}. The high-level controller estimates the desired locomotion mode, which is utilized to trigger the switching of the control parameters in the middle-level controller. The control parameters may include joint impedance parameters \citep{sup_design_2008}, swing trajectories \citep{mendez_powered_2020}, and the assistive force curve \citep{ding_human---loop_2018}. The low-level controller then activates joint torque control or position control based on the control parameters from the middle-level controller. The control parameters of different locomotion modes can be optimized and preprogrammed for different subjects \citep{zhang_human---loop_2017, ding_human---loop_2018}. The remaining key problem is how and when to properly switch between different locomotion modes (e.g., standing, walking, running, and climbing stairs). To switch the locomotion modes smoothly, wearable robots should first predict the wearer's intent in a timely manner \citep{xu_real-time_2018}.

Human intent is the mental activity that indicates the future actions of humans and is difficult to predict. Human-robot interface signals, including the electromyograph (EMG) signals \citep{clites_proprioception_2018, hu_fusion_2018} and the body motion as estimated by an inertial measurement unit (IMU) \citep{xu_real-time_2018, chen2021probability}, are usually preferred for decoding the human intent, as considered in previous research. Many classification methods have been proposed to classify the measured human-robot signals and estimate the locomotion mode of humans through that information  \citep{xu_real-time_2018, hu_fusion_2018,  hu_deep_2019}. Existing classification methods can achieve high classification accuracy on a given dataset \citep{hu_fusion_2018} but may provide an inferior performance on a new dataset, particularly if there is a distribution discrepancy between the two datasets. The distribution discrepancy, however, is common for human-robot interface signals. 
The skin, muscles, and locomotion patterns of different subjects are different, which makes the human-robot interface signals user-dependent. Even for the same subject, the interface signals may still vary depending on the sensor position and the subject's physical state (e.g. limb fatigue and skin conductivity). To accurately estimate the locomotion intent of a human, the existing methods have to collect and label a large number of signals and need to train a new classifier for each new subject \citep{hu_deep_2019}, which is infeasible.

One solution to the user-dependency problem is to adopt user-independent sensors, such as vision sensors. Vision sensors can perceive the environment ahead of time and thus predict the locomotion mode of humans at reasonable accuracy \citep{krausz_depth_2015, massalin_user-independent_2018, zhang_environmental_2019, zhong_environmental_2020}. Our previous research \citep{zhang_subvision_2021} has 
demonstrated that a vision system trained with the dataset of one healthy subject can be directly applied to accurately adjust the control modes of a powered prosthesis and assist different transfemoral amputees to walk on complex terrains without the need for training. These results have validated that a vision system is typically user-independent. Nevertheless, vision sensors may not recognize the state transitions induced by a human who wears the device, such as transitioning between standing and walking. Therefore, the human-robot interface signals still seem necessary to fully decode the human intent.

There is always a trade-off in the sensor choice: accurately decoding human intent from interface signals is important but the problem of user-dependency cannot be avoided. How can we decrease the user dependency of the interface signals? An important cause of user dependency is the distribution discrepancy of the interface signals from different subjects. If the distribution discrepancy can be decreased, the user dependency can also be decreased. A common method for resolving this problem is domain adaption \citep{ganin_domain-adversarial_2016, saito_maximum_2018, chen_mutual_2021}. In this method, there are two types of data: source-domain data and target-domain data. The source-domain data are labeled and the target-domain data are unlabeled, and there is a discrepancy between the two domains. The unsupervised domain adaptation methods train the classifier using the labeled source-domain data and unlabeled target-domain data. After training, the trained classifier can be used to classify the target-domain data with acceptable accuracy. Two typical domain adaptation methods are domain-adversarial neural networks (DANN) \citep{ganin_domain-adversarial_2016} and maximum classification discrepancy (MCD) \citep{saito_maximum_2018}. DANN was proposed by 
Ganin \etal and contains a feature generator, a domain classifier, and a label classifier. The domain classifier and the feature generator are adversarial. The domain classifier aims to distinguish between the features extracted from the source-domain data and those from the target-domain data. The feature generator seeks to extract similar hidden features from the source domain and the target domain data, to confuse the domain classifier until the domain classifier is unable to distinguish these features. Also, the feature generator requires to extract important features that can be classified accurately by the label classifier. DANN works well but it does not consider the classes of features. The discrepancy between source and target data may not be consistent in different classes, and thus sometimes the global adaptation may fail. To address this issue, Saito \etal proposed MCD \citep{saito_maximum_2018}, which consists of one feature generator and two classifiers. The two classifiers and the feature generator are adversarial. The classifiers are trained to maximize the discrepancy between them while the feature generator is trained to minimize the classifier discrepancy. After training, MCD is known to outperform DANN because it aligns the feature of two domains in each class. 

Existing domain adaptation methods are typically utilized to classify images. Few domain adaptation studies have focused on classifying human-robot interface signals and predicting the locomotion intent of humans. In our previous study, an unsupervised cross-subject adaptation method similar to MCD was proposed \citep{zhang_unsupervised_2020}. After training the designed convolutional neural network by the labeled data from the source subjects and the unlabeled data from the target subjects, a convolutional neural network was able to predict the locomotion intent of the target subject. Our past research validated the effectiveness of the unsupervised domain adaptation method in predicting the locomotion intent of humans, for the first time. Although the results in that research look promising, some limitations and challenges remain. 
\textcolor{black}{
First, the single feature generator of MCD may only learn a subset of the features from the input data when the input data are multi-view and contain multiple hidden features \citep{allen-zhu_towards_2021, zhao_multi-view_2017}, which is theoretically demonstrated in a thought experiment of this paper.
Ensemble learning \citep{bolon-canedo_ensembles_2019, ali_smart_2020} is able to address the mentioned issue by extracting diverse features and increasing the generalization ability, but with the sacrifice of efficiency. However, efficiency is also critical for wearable robots. Is it possible to learn diverse features and remain efficient simultaneously?} 
Moreover, a labeled validation set of the target subject was required to evaluate the performance of the neural network in the training process and determine the time to stop training. Will the network overfit the dataset of the source subjects if there is no labeled validation set of the target subject?

\textcolor{black}{In order to deal with these issues, the present paper develops a new unsupervised cross-subject adaptation method, which is called ensemble diverse hypotheses and knowledge distillation (EDHKD). The proposed EDH includes multiple diverse hypotheses and each hypothesis consists of a feature generator and a classifier, which is inspired by ensemble learning \citep{zhou_ensemble_2021}, whose principle is based on the use of two intelligent agents rather than one. 
We suppose that the EDH can learn diverse features by maximizing the discrepancy of different feature generators. Then EDHKD, which is the student of EDH and only includes one feature generator and one classifier, can learn the knowledge from EDH by learning the pseudo labels of EDH. Hence, the EDHKD can increase the generalization ability and remain efficient simultaneously. Moreover, the present paper also solves the overfitting problem by designing a lightweight backbone neural network. The overview of the proposed EDHKD is shown in \autoref{fig:1_cross_subject_overview}. The key contributions of the present paper are the following:
}
\begin{enumerate}
    \item \textcolor{black}{Theoretically explaining why the single feature generator of MCD is prone to learn a subset of features from multi-view data and how the ensemble learners and knowledge distillation learn the full set of the features.}
    \item \textcolor{black}{Combining diverse hypotheses, optimizing the loss functions, and distilling the knowledge to accurately and efficiently predict the locomotion modes of the target subject whose signals are not labeled.}
    \item Optimizing the network structure to avoid overfitting without requiring any labeled data of the target subjects.
    \item Achieving state-of-the-art accuracy levels ( 94.4\% and 97.4\%) on two public datasets for classifying the locomotion intent and activities of the target subject.
\end{enumerate}

\begin{figure*}[h!]
    \centering
    \includegraphics[width=\textwidth]{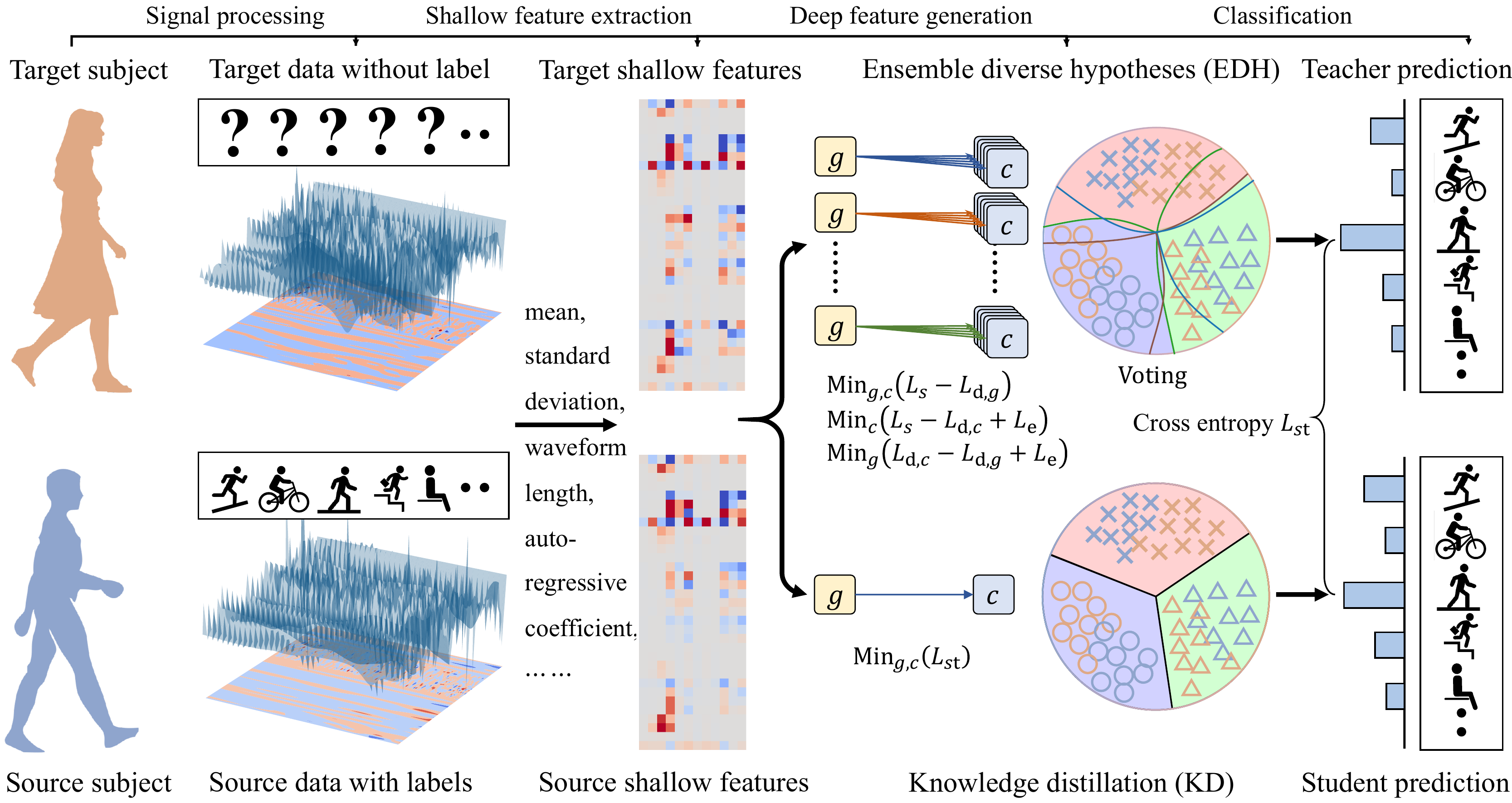}
    \caption{Overview of the ensemble diverse hypotheses and knowledge distillation (EDHKD). Shallow features, such as the mean and the standard deviation, are first extracted from the original human-robot interface signals. The target features are first recognized by the teacher network called ensemble diverse hypotheses (EDH). EDH is trained to minimize the source error ($L_s$), maximize the feature discrepancy ($L_{d,g}$), minimize the upper bound of classifier discrepancy ($L_{d,c}$), and minimize the entropy of the output predictions ($L_e$). After training, 
    the teach predictions of EDH, which can be regarded as pseudo labels, are utilized to train the student network (EDHKD). Both EDH and EDHKD can accurately predict the locomotion modes of target subjects but the latter is much more efficient.}
    \label{fig:1_cross_subject_overview} 
\end{figure*}

\section{Materials and Methods}
\label{sec:Cross-subjectMethods}

\subsection{\textcolor{black}{Thought experiment of the single learner, ensemble learners, and knowledge distillation}}
\label{subsec:thought_experiment}
Before introducing the details of the proposed EDHKD on classifying human-robot interface signals, the present section constructs a thought experiment to theoretically demonstrate the limitations of the single learner adopted in MCD and how the ensemble learners and knowledge distillation address the limitations. As introduced in \citep{allen-zhu_towards_2021}, the real data are usually ``multi-view", which indicates that the data may contain multiple hidden features that can be utilized to classify the input data. \autoref{theorem:single_learner} displays that the single learner may only learn a subset of the hidden features from the multi-view data, which degrades the generalization ability of the learner when there is a discrepancy between the source and target domain.

\begin{theorem}
\label{theorem:single_learner}
(Single learner)
Let $g$, $c_1$, and $c_2$ be a feature generator and two linear binary classifiers. Let $\bm{x}_s$ and $\bm{x}_t$ be the source data and target data, respectively. The label of $\bm{x}_s$ and $\bm{x}_t$ are $y_s \in \{0, 1\}$ and $y_t \in \{0, 1\}$. 
Assume that three orthogonal hidden features $\bm{v}_1$, $\bm{v}_2$, $\bm{v}_3$, $\parallel\bm{v}_1\parallel = \parallel\bm{v}_2\parallel = \parallel\bm{v}_3\parallel = 1$ exist in the all source data $\bm{x}_s$. The linear combination of the hidden features $k_1\bm{v}_1 + k_2\bm{v}_2 + k_3\bm{v}_3$ can be extracted from the feature generator $g$ through the supervised learning. Here $\{k_i \geq 0 |i \in \{ 1, 2, 3\}\}$ are the coefficients of different hidden features. The activation function of classifier 
is set as ReLU and thus $c(g(\bm{x})) \geq 0$. The prediction $\hat{y}=1$ if $c(g(\bm{x})) > 0$, else $\hat{y}=0$.
Because the source domain and the target domain are different, only a part of target data (ratio = $\alpha$, $\alpha < 1$) contain all three hidden features $\bm{v}_1$, $\bm{v}_2$, and $\bm{v}_3$. Assume $\frac{1-\alpha}{3}$ target data only contain one hidden feature $\bm{v}_1$, $\bm{v}_2$, or $\bm{v}_3$, respectively. After training $c_1$ and $c_2$ to maximize their discrepancy (L1 distance) and training $g$ to minimize the classifier discrepancy (L1 distance) in the target domain, there is a probability that $g$ can only learn a subset of features, which means $\exists k \in \{k_1, k_2, k_3\}, k = 0$.
\end{theorem}

\begin{proof}
Since $c_1$ and $c_2$ are linear:
\begin{equation}
\label{eq:linear_constraint}
\begin{split}
c_1(g(\bm{x})) &= c_1(k_1\bm{v}_1 + k_2\bm{v}_2 + k_3\bm{v}_3) = k_1c_1(\bm{v}_1) + k_2c_1(\bm{v}_2) + k_3c_1(\bm{v}_3)\\
c_2(g(\bm{x})) &= c_2(k_1\bm{v}_1 + k_2\bm{v}_2 + k_3\bm{v}_3) = k_1c_2(\bm{v}_1) + k_2c_2(\bm{v}_2) + k_3c_2(\bm{v}_3)\\
\end{split}
\end{equation}

Then the classifier discrepancy in the target domain is:
\begin{equation}
\label{eq:classifer_discrepancy}
\begin{split}
\epsilon_{T_g}(c_1, c_2) &= \bm{E}_{g(\bm{x}_t) \sim T_g}\big[|c_1(g(\bm{x}_t)) - c_2(g(\bm{x}_t))|\big]\\
&= \alpha |c_1(k_1\bm{v}_1 + k_2\bm{v}_2 + k_3\bm{v}_3) - c_2(k_1\bm{v}_1 + k_2\bm{v}_2 + k_3\bm{v}_3)| \\
&+ \frac{1-\alpha}{3}\big[|c_1(k_1\bm{v}_1)-c_2(k_1\bm{v}_1)| + |c_1(k_2\bm{v}_2)-c_2(k_2\bm{v}_2)| \\
&+ |c_1(k_3\bm{v}_3)-c_2(k_3\bm{v}_3)|\big]\\
&= (\alpha + \frac{1-\alpha}{3})\big[k_1|c_1(\bm{v}_1)-c_2(\bm{v}_1)| + k_2|c_1(\bm{v}_2)-c_2(\bm{v}_2)| \\
&+ k_3|c_1(\bm{v}_3)-c_2(\bm{v}_3)|\big] \\
&\leq \frac{1+2\alpha}{3} \big[k_1|c_1(\bm{v}_1)| +k_1|c_2(\bm{v}_1)| + k_2|c_1(\bm{v}_2)|+k_2|c_2(\bm{v}_2)| \\
&+ k_3|c_1(\bm{v}_3)|+k_3|c_2(\bm{v}_3)|\big] \big]
\end{split}
\end{equation}
where $g(\bm{x}_t) = k_i \bm{v}_i, i\in\{1, 2, 3\}$ if $\bm{x}_t$ only contains one feature $\bm{v}_i$. The maximum value is achieved if and only if $\{c_1(\bm{v}_i) c_2(\bm{v}_i) \leq 0 | i\in\{1, 2, 3\}\}$.

To maximize the $\epsilon_{T_g}(c_1, c_2)$, $c_1(\bm{v}_i) c_2(\bm{v}_i) \leq 0$ for every $i \in \{1, 2, 3\}$. Since $c_j(\bm{v}_i) \geq 0, j \in \{1, 2\}$:
\begin{equation}
\label{eq:zero_classification}
\begin{split}
\exists c_j, c_j(\bm{v}_i) = 0, j \in \{1, 2\}, i \in \{1, 2, 3\}.
\end{split}
\end{equation}

Because $c_1$ and $c_2$ and trained to classify source data correctly:
\begin{equation}
\label{eq:supervised_constraints}
\begin{split}
c_1(g(\bm{x}_s))&=c_1(k_1\bm{v}_1 + k_2\bm{v}_2 + k_3\bm{v}_3)\\
&= k_1c_1(\bm{v}_1) + k_2c_1(\bm{v}_2) + k_3c_1(\bm{v}_3)> 0\\
c_2(g(\bm{x}_s))&=c_2(k_1\bm{v}_1 + k_2\bm{v}_2 + k_3\bm{v}_3)\\
& = k_1c_2(\bm{v}_1) + k_2c_2(\bm{v}_2) + k_3c_2(\bm{v}_3)> 0
\end{split}
\end{equation}

Combining \eqref{eq:zero_classification} and \eqref{eq:supervised_constraints}, it can be learned:
\begin{equation}
\label{eq:valid_feature_num}
\begin{split}
\#N_1 + \#N_2 \in \{3, 4\}, \text{for }i_1 \in N_1\text{ and }i_2 \in N_2, c_1(\bm{v}_{i_1}) = c_2(\bm{v}_{i_2}) = 0 
\end{split}
\end{equation}
where $\#N_1$ and $\#N_2$ indicate the number of items in the set $N_1$ and $N_2$, respectively.

If $\#N_1 + \#N_2 = 3$, without loss of generality, 
we can assume $c_1(\bm{v}_1) = 1, c_1(\bm{v}_2)= c_1(\bm{v}_3)= 0 $ while $c_2(\bm{v}_1) = 0, c_2(\bm{v}_2)= c_2(\bm{v}_3) = 1$.

Because of \eqref{eq:linear_constraint} and \eqref{eq:supervised_constraints}:
\begin{equation}
\label{eq:k_constraint}
\begin{split}
&c_1(g(\bm{x})) = k_1c_1(\bm{v}_1) + k_2c_1(\bm{v}_2) + k_3c_1(\bm{v}_3) = k_1 > 0 \\
&c_2(g(\bm{x})) = k_1c_2(\bm{v}_1) + k_2c_2(\bm{v}_2) + k_3c_2(\bm{v}_3) = k_2 + k_3 > 0\\
\end{split}
\end{equation}

The last step of MCD is to train the feature generator to minimize the classifier discrepancy: 
\begin{equation}
\label{eq:supervised_constraint_1}
\begin{split}
\min_{k_1, k_2, k_3} \epsilon_{T_g}(c_1, c_2) &= \min_{k_1, k_2, k_3} \frac{1+2\alpha}{3}\big[|k_1c_1(\bm{v}_1)-k_1c_2(\bm{v}_1)| \\ &+ |k_2c_1(\bm{v}_2)-k_2c_2(\bm{v}_2)| + |k_3c_1(\bm{v}_3)-k_3c_2(\bm{v}_3)|\big] \\
&=\min_{k_1, k_2, k_3}\frac{1+2\alpha}{3}\big[k_1 + k_2 + k_3\big]
\end{split}
\end{equation}

Because $k_1 + k_2 + k_3 \geq k_1 + k_2$ and $k_1 + k_2 + k_3 \geq k_1 + k_3$ as $\{k_i \geq 0 | i\in\{1, 2, 3\}\}$ and because of \eqref{eq:k_constraint} and \eqref{eq:supervised_constraint_1}, the condition of minimizing the classifier discrepancy is that $\exists i \in \{2, 3\}, k_i = 0$. 

Since $\exists i \in \{2, 3\}, k_i = 0$, $c(g(\bm{x}_t)) = 0$ if $\bm{x}_t$ only contains feature $\bm{v}_i, i \in \{2, 3\}$. Then the target accuracy is $1 - \frac{1-\alpha}{3}$.   

According to the symmetry, we can find that the training result of MCD for the single model is that $g$ can only learn a subset of features: $\exists k \in \{k_1, k_2, k_3\}, k = 0$, if a classifier can recognize two features. In this case, the target accuracy is $1 - \frac{1-\alpha}{3}$.

If $\#N_1 + \#N_2 = 2$, without loss of generality, 
we can assume $c_1(\bm{v}_1) = 1, c_1(\bm{v}_2)= c_1(\bm{v}_3)= 0 $ while $c_2(\bm{v}_1) = c_2(\bm{v}_3) = 0, c_2(\bm{v}_2) = 1$.

In this case, minimizing the classifier discrepancy equals: 
\begin{equation}
\begin{split}
\min_{k_1, k_2, k_3} \epsilon_{T_g}(c_1, c_2) &= \min_{k_1, k_2, k_3} \frac{1+2\alpha}{3}\big[|k_1c_1(\bm{v}_1)-k_1c_2(\bm{v}_1)| \\ &+ |k_2c_1(\bm{v}_2)-k_2c_2(\bm{v}_2)| + |k_3c_1(\bm{v}_3)-k_3c_2(\bm{v}_3)|\big] \\
&=\min_{k_1, k_2, k_3}\frac{1+2\alpha}{3}\big[k_1 + k_2\big]
\end{split}
\end{equation}
where the $k_3$ does not affect the result and thus $k_3$ can be either 1 or 0, and the probability of each case is 50\%.

Then $c(g(\bm{x}_t)) = 0$ if $\bm{x}_t$ only contains feature $\bm{v}_3$. The target accuracy will be $1 - \frac{1-\alpha}{3}$.

According to the above two situations, there is a high probability ($p > 50\%$) that the single feature generator only learns a subset of features. Moreover, the target accuracy in this theoretical experiment is $1 - \frac{1-\alpha}{3}$ for the single learner.
\end{proof}

The single learner may only learn a subset of features. On the contrary, ensemble multiple learners may learn all features. In the \autoref{theorem:ensemble_learners}, we will demonstrate that maximizing the feature diversity allows feature generators to learn all features.

\begin{theorem}
\label{theorem:ensemble_learners}
(Ensemble learners) Let $g_1$, $g_2$, and $g_3$ be three feature generators. Each feature generator is able to learn the linear combination of three orthogonal hidden features $k_{i1}\bm{v}_1 + k_{i2}\bm{v}_2 + k_{i3}\bm{v}_3$, $\parallel\bm{v}_1\parallel = \parallel\bm{v}_2\parallel = \parallel\bm{v}_3\parallel = 1$. After maximizing the diversity (L1 distance) of features: 
\begin{equation}
\begin{split}
L_{d,g} = \bm{E}_{\bm{x}}\Big[\sum_{i=1}^{3}\lvert g_i(\bm{x}) - \bm{E}_{i\in[1,3]} g_i(\bm{x})\lvert \Big].
\end{split}
\end{equation}
the ensemble of three learners will learn all hidden features:
\begin{equation}
\begin{split}
\exists i, k_{ij} > 0, i\in\{1, 2, 3\} \text{ for every } j \in \{1, 2, 3\}.
\end{split}
\end{equation}
\end{theorem}

\begin{proof}
According to \autoref{theorem:single_learner}, the discrepancy of features:
\begin{equation}
\begin{split}
L_{d,g} &= r_s \bm{E}_{\bm{x}_s}\Big[\sum_{i=1}^{3}\lvert g_i(\bm{x}_s) - \bm{E}_{i\in[1,3]} g_i(\bm{x}_s)\lvert \Big] \\ 
&+ r_t \bm{E}_{\bm{x}_t}\Big[\sum_{i=1}^{3}\lvert g_i(\bm{x}_t) - \bm{E}_{i\in[1,3]} g_i(\bm{x}_t)\lvert \Big]\\
&= (r_s + r_t\alpha)\sum_{i=1}^{3}\lvert \sum_{j=1}^{3} (k_{ij}-\bm{E}_{i\in[1,3]}k_{ij})\bm{v}_j\lvert \\
&+ r_t\frac{1-\alpha}{3}\sum_{i=1}^{3} \sum_{j=1}^{3} \lvert(k_{ij}-\bm{E}_{i\in[1,3]}k_{ij})\bm{v}_j\lvert
\end{split}
\end{equation}
where $r_s$ and $r_t$ are the ratio of the sample number of source data and the target data to the total sample number, respectively.

For $L_1$ (Manhattan) distance:
\begin{equation}
\begin{split}
\lvert \sum_{i=1}^{3} k_i \bm{v}_i\lvert = \sum_{i=1}^{3} \big\lvert k_i \parallel \bm{v}_i\parallel \big\lvert=
\sum_{i=1}^{3} \lvert k_i\lvert.
\end{split}
\end{equation}

Hence:
\begin{equation}
\begin{split}
L_{d,g} &= (r_s + r_t\alpha)\sum_{i=1}^{3} \sum_{j=1}^{3} \lvert k_{ij}-\bm{E}_{i\in[1,3]}k_{ij}\lvert \\
&+ r_t\frac{1-\alpha}{3}\sum_{i=1}^{3} \sum_{j=1}^{3} \lvert k_{ij}-\bm{E}_{i\in[1,3]}k_{ij}\lvert
\\&= (r_s + r_t\frac{1+2\alpha}{3})\sum_{i=1}^{3} \sum_{j=1}^{3} \lvert k_{ij}-\bm{E}_{i\in[1,3]}k_{ij}\lvert
\end{split}
\end{equation}

Without loss of generality, we can assume $k_{1j} \leq k_{2j} \leq k_{3j}$. Then:
\begin{equation}
\begin{split}
\sum_{i=1}^{3} \lvert k_{ij}-\bm{E}_{i\in[1,3]}k_{ij}\lvert &= \bm{E}_{i\in[1,3]}k_{ij} - k_{1j} + k_{3j} - \bm{E}_{i\in[1,3]}k_{ij} + \lvert k_{2j} - \bm{E}_{i\in[1,3]}k_{ij} \lvert\\
&= k_{3j} - k_{1j} + \lvert k_{2j} - \frac{k_{1j} + k_{2j} + k_{3j}}{3}\lvert\\
&=k_{3j} - k_{1j} + \lvert \frac{2}{3}(k_{2j} - \frac{k_{1j}+k_{3j}}{2})\lvert \\
&\leq \frac{4}{3}(k_{3j} - k_{1j})
\end{split}
\end{equation}
where the maximum value is achieved when $k_{2j} = k_{3j}$ or $k_{2j} = k_{1j}$ and $k_{1j} = 0$.

Therefore, maximizing the feature discrepancy will cause max$\{k_{ij} | i \in\{1, 2, 3\}\} > 0$ and min$\{k_{ij} | i \in\{1, 2, 3\}\} = 0$. For every $j \in \{1, 2, 3\}, \exists i = \argmax_{i \in\{1, 2, 3\}} k_{ij}, k_{ij} > 0$. Hence, all three hidden features will be learned. 

According to \eqref{eq:k_constraint}, each feature generator should learn at least two hidden features. Therefore, for each $\bm{x}_t$ that only contains one feature $\bm{v}_l$, $\exists i = \argmin_{i \in\{1, 2, 3\}}, k_{ij} = 0|j\neq l$, and $g_i(\bm{x}_s) = k_{il}\bm{v}_l + k_{im} \bm{v}_m, m \neq j$. 

Because the classifiers are trained to maximize the classifier discrepancy, which is shown in \eqref{eq:classifer_discrepancy}, $c_{i1}(\bm{v}_m) c_{i2}(\bm{v}_m) \leq 0$. Therefore, $\exists n \in \{1, 2\}, c_{in}{\bm{v}_m} = 0$. According to \eqref{eq:supervised_constraints}, the source data should be classified accurately: 
\begin{equation}
\begin{split}
&c_{in}(g_i(\bm{x}_s)) > 0 \\
&c_{in}(k_{il}\bm{v}_l + k_{im}\bm{v}_m) > 0 \\
&k_{il}c_{in}\bm{v}_l + 0 > 0 \\
&c_{in}\bm{v}_l > 0
\end{split}
\end{equation}

In conclusion, for each $\bm{x}_t$ which contains feature $\bm{v}_l$, $\exists i \in \{1, 2, 3\} \text{ and } n\in\{1, 2\}$, $c_{in}(g_i(\bm{x}_t)) = k_{il}c_{in}(\bm{v}_l) > 0$. Therefore, the target accuracy for the ensemble learners in this theoretical experiment is 100\%.
\end{proof}

\begin{theorem}
\label{theorem:knowledge_distillation}
(Knowledge distillation) Let $\{g_i, c_{in}| i \in \{ 1, 2, 3\}, n\in \{1, 2\}\}$, be the teacher ensemble learners mentioned \autoref{theorem:ensemble_learners}. A student learner with only one feature generator $g_s$ and one classifier $c_s$ can also classify all target accurately $g_s(\bm{x}) =  k_{s1}\bm{v}_1 + k_{s2}\bm{v}_2 + k_{s3} \bm{v}_3 > 0$, which indicates $\{k_{sj} > 0 | j\in \{ 1, 2, 3\}\}$ and $\{c_s(\bm{v}_j) > 0| j \in \{ 1, 2, 3\}\}$, by learning the pseudo labels of the teach network.
\end{theorem}

\begin{proof}
For $\bm{x}_t$ that only contains feature $\bm{v}_j$, the student learner is trained to by the soft label of the teacher learner. As proved in \autoref{theorem:ensemble_learners}, $\exists i \in \{1, 2, 3\}$ and $n \in \{1, 2\}$,
$c_{in}(g_i(\bm{x}_t)) > 0$. Then the soft label, which is average of ensemble learners, $\bar{y}(\bm{x}_t) > 0$. After minimizing the cross entropy between $\bar{y}(\bm{x}_t)$ and $c_s(g_s(\bm{x}_t))$, $c_s(g_s(\bm{x}_t)) > 0$.

Because $\bm{x}_t$ only contains feature $\bm{v}_j$,
\begin{equation}
\begin{split}
&c_s(g_s(\bm{x}_t)) = c_s(k_{sj}\bm{v}_j) = k_{sj}
c_s(\bm{v}_j) > 0 \\
&\rightarrow k_{sj} > 0 \text{ and } c_s(\bm{v}_j) > 0, j \in \{1, 2, 3\}.
\end{split}
\end{equation}
\end{proof}

The above theorems illustrate the limitations of the single learner and the advantages of the ensemble learner and knowledge distillation in a theoretical situation. The following sections will further introduce the strength of the EDHKD and its implementation details.

\subsection{General theoretical bases of EDHKD}
\label{subsec:theoretical_basis}
\textcolor{black}{
Ensemble diverse hypotheses and knowledge distillation (EDHKD) is proposed in the present paper because EDHKD may learn diverse features and theoretically decrease the upper bound of the error of predicting the intent of a target subject. As introduced in \autoref{subsec:thought_experiment}, the single learner may only learn a subset of features. Besides, the number of parameters for a neural network is large and the loss function of the unsupervised cross-subject adaptation may not be convex, training the neural network does not guarantee finding the global minimum. This issue may be alleviated by combining many learners. However, the combination of learners is only helpful if they provide different predictions, 
which is the first reason for designing EDH. Each hypothesis may make a mistake but different hypotheses may make different mistakes. After combining the predictions of all hypotheses, some mistakes may be resolved. The error of the ensemble prediction can be low if each hypothesis is accurate and all hypotheses are diverse. 
}
\begin{theorem}
\label{theorem:error_discrepancy_decomposition}
(error-discrepancy decomposition)
Let $\bar{h}$ be the mean of all individual hypotheses $h_k$ and let $e(h_k)$, $e(\bar{h})$, and $d({h_k})$ be the error of the hypothesis $h_k$, the error of the average hypothesis $\bar{h}$, and the discrepancy between the hypothesis $h_k$ and the average hypothesis $\bar{h}$. $e(\bar{h})$ equals:
\begin{equation}
\label{eq:error-discrepany_decompostion}
\begin{split}
&e(\bar{h}) = \bm{E}_{k}e(h_k) - \bm{E}_{k}d({h_k}),\\
\end{split}
\end{equation}
where $\bm{x}$ and $y$ are input data and label. $\bm{E}$ indicates the expectation. Besides,
\begin{equation}
\label{eq:error-discrepany_decompostion_items}
\begin{split}
&e(\bar{h}) = \bm{E}_{\bm{x}}(\bar{h}(\bm{x})-y)^2,\\
&e(h_k) = \bm{E}_{\bm{x}}(h_k(\bm{x})-y)^2,\\
&d({h_k}) = \bm{E}_{\bm{x}}(h_k(\bm{x})-\bar{h}(\bm{x}))^2.\\
\end{split}
\end{equation}
\end{theorem}

\begin{proof}
\begin{align*}
\bm{E}_{k}d({h_k}) 
&=\bm{E}_{k}\bm{E}_{\bm{x}}(h_k(\bm{x})-\bar{h}(\bm{x}))^2\\
&=\bm{E}_{k}\bm{E}_{\bm{x}}[h_k(\bm{x}) - y - (\bar{h}(\bm{x}) - y)]^2\\
&=\bm{E}_{k}\bm{E}_{\bm{x}}[(h_k(\bm{x}) - y)^2 + (\bar{h}(\bm{x}) - y)^2]\\
&-2\bm{E}_{k}\bm{E}_{\bm{x}}[(h_k(\bm{x}) - y)(\bar{h}(\bm{x}) - y)]\\
&=\bm{E}_{k}\bm{E}_{\bm{x}}[(h_k(\bm{x}) - y)^2 -(\bar{h}(\bm{x}) - y)^2]\\
&=\bm{E}_{k}\bm{E}_{\bm{x}}(h_k(\bm{x}) - y)^2 - \bm{E}_{\bm{x}}(\bar{h}(\bm{x}) - y)^2\\
&=\bm{E}_{k}e(h_k) - e(\bar{h}).
\end{align*}
Therefore, $e(\bar{h}) = \bm{E}_{k}e(h_k) - \bm{E}_{k}d({h_k})$.
\end{proof}
\autoref{theorem:error_discrepancy_decomposition} indicates that the error of the average hypothesis is lower than the average error of the individual hypotheses if the discrepancy of hypotheses is not zero. The error of the average hypothesis will further decrease if the error of individual hypotheses decreases and the discrepancy of the hypotheses increases. EDH can improve the accuracy of predicting the subject intent because it may increase the discrepancy between different hypotheses. Besides, EDH may decrease the error of the individual hypothesis $e(h_k)$ in the target domain, which is the second reason for proposing EDH. It should be noted that the discrepancy of hypotheses is convenient to estimate because it does not require the labels of data, but the error of the individual hypotheses in the target domain, where the target labels are not available, is hard to estimate. The present paper addresses this issue by estimating the upper bound of the error of the individual hypotheses in the target domain.

\begin{theorem}
\label{theorem:upper_bound_of_target_error}
(Upper bound of the target error)
Let $h$ be a hypothesis that consists of a feature generator $g$ and a classifier $c$: $h = c \circ g$. Let $C$ be the classifier space and let $S_g$ and $T_g$ be the transformed source and target domain using the feature generator $g$. The target classification error $\epsilon_{T_g}(c)$ is upper bounded by:
\begin{equation}
\label{eq:sup_error}
\begin{split}
&\epsilon_{T_g}(c) \leq \epsilon_{S_g}(c) + \lambda + {d_{C \Delta C}(S_g,T_g)}/{2},\\
&\lambda = \min_{c\in C}\big(\epsilon_{S_g}(c) + \epsilon_{T_g}(c)\big),\\
&d_{C \Delta C}(S_g,T_g) = 2 \sup_{c_1,c_2 \in C} |\epsilon_{T_g}(c_1, c_2) - \epsilon_{S_g}(c_1, c_2)|,\\
& \epsilon_D(c) = \bm{E}_{g(\bm{x}) \sim D}\big[|c(g(\bm{x}))-y|\big],\\
& \epsilon_D(c_1, c_2) = \bm{E}_{g(\bm{x}) \sim D}\big[|c_1(g(\bm{x})) - c_2(g(\bm{x}))|\big],
\end{split}
\end{equation}
where $\epsilon_{D}(c)$ and $\epsilon_D(c_1, c_2)$ are classification error of the classifier $c$ and the discrepancy of two classifiers $c_1$ and $c_2$ in the domain $D$ ($D = S_g$ or $T_g$), respectively. $d_{C \Delta C}(S_g,T_g)$ indicates the $C \Delta C$-distance, $\lambda$ denotes the combined error in the transformed source and target domain, $\sup$ indicates the upper bound, and $\bm{E}$ 
represents the expectation.
\end{theorem}
\begin{proof}
Ben-David \etal \citep{ben-david_analysis_2007} proved a triangle inequality of classification error: for any classification functions $c_1$,$c_2$, and $c_3$, $\epsilon(c_1, c_2) \leq \epsilon(c_1, c_3)+\epsilon(c_2, c_3)$. Therefore,
\begin{align*}
\epsilon_{T_g}(c)
&\leq \epsilon_{T_g}(c^*) + \epsilon_{T_g}(c, c^*)\\
&\leq \epsilon_{T_g}(c^*) + \epsilon_{S_g}(c, c^*) + |\epsilon_{T_g}(c, c^*) - \epsilon_{S_g}(c, c^*)|\\
&\leq \epsilon_{T_g}(c^*) + \epsilon_{S_g}(c) + \epsilon_{S_g}(c^*) + |\epsilon_{T_g}(c, c^*) - \epsilon_{S_g}(c, c^*)|\\
&\leq \epsilon_{S_g}(c) + \lambda + |\epsilon_{T_g}(c, c^*) - \epsilon_{S_g}(c, c^*)|\\
&\leq \epsilon_{S_g}(c) + \lambda + \sup_{c_1,c_2 \in C}|\epsilon_{T_g}(c_1, c_2)-\epsilon_{S_g}(c_1, c_2)|\\
&= \epsilon_{S_g}(c) + \lambda + {d_{C \Delta C}(S_g,T_g)}/{2},
\end{align*}
where $c^* = \argmin_{c\in C}\big(\epsilon_{S_g}(c) + \epsilon_{T_g}(c)\big)$.
\end{proof}

\begin{figure*}[h!]
    \centering
    \includegraphics[width=\textwidth]{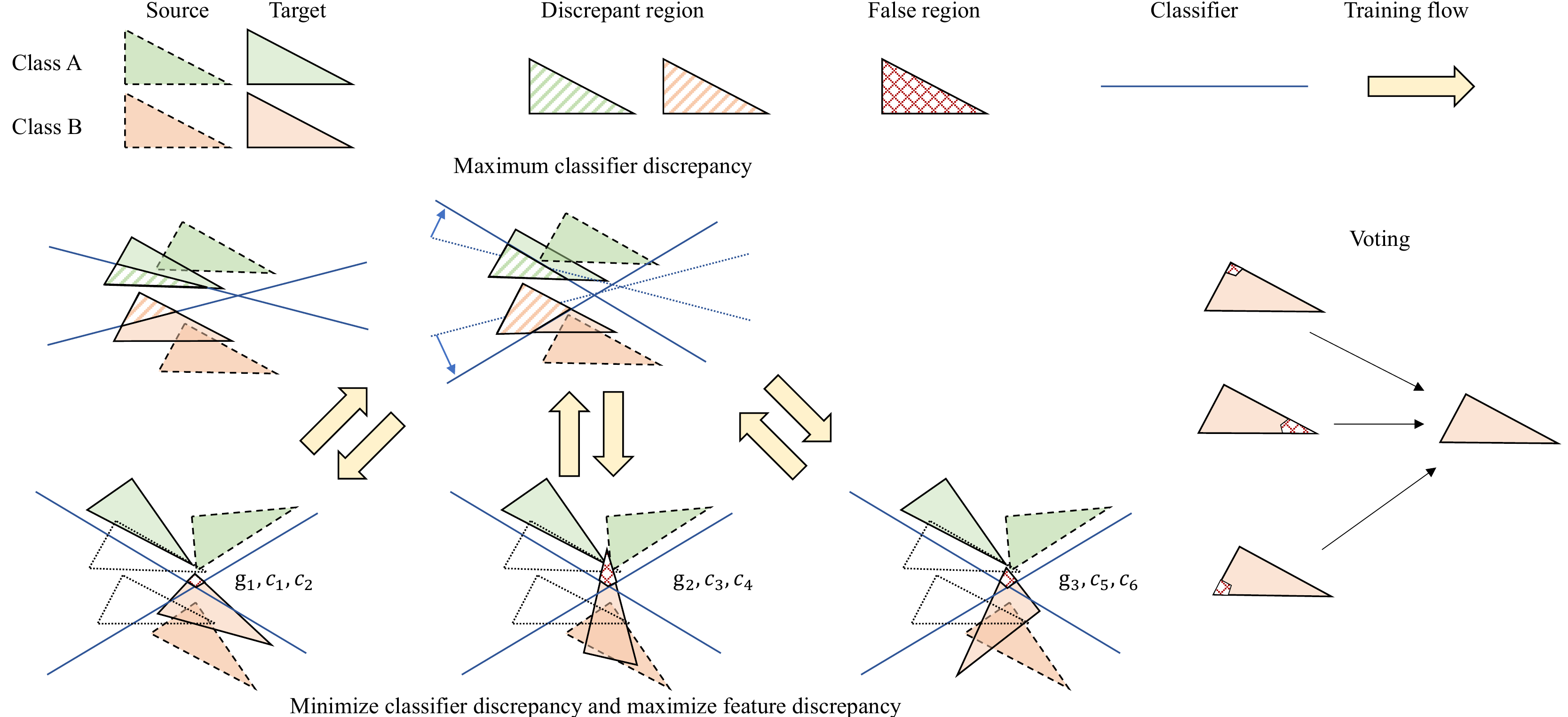}
    \caption{Illustration of the EDH principle. In this figure, EDH consists of three diverse feature generators ($g_1,g_2, g_3$) and two diverse classifiers ($c_1,...,c_6$) for each feature generator. The classifiers are first trained to maximize the discrepancy region to find the upper bound of the classifier discrepancy. Then three diverse feature generators are trained to minimize the classifier discrepancy and maximize the feature discrepancy. After adaptation, the classifier discrepancy can be minimized but two classifiers may misclassify the same region. Because the diverse feature generators may transform the target data into different feature spaces, the false regions for different feature spaces may be different. After voting, the target data can still be accurately recognized.}
    \label{fig:EDH_principle} 
\end{figure*}

\autoref{theorem:error_discrepancy_decomposition} and \autoref{theorem:upper_bound_of_target_error} demonstrate that EDH can decrease the upper bound of the error for the hypothesis in the target domain by reducing the source error, the combined error, and the $C \Delta C$-distance, and increasing the discrepancy of hypotheses. The corresponding loss functions that may decrease the target error can be derived. In the following parts of this section, several proposed loss functions and their corresponding theoretical bases will be introduced in detail.

\subsubsection{Minimization of the Source Error}
\label{subsubsec:source_error}
\autoref{theorem:upper_bound_of_target_error} indicates that the target error is positively correlated with the source error. It is convenient to decrease the error on the source domain by  minimizing the cross-entropy between the classification results and the labels $y_s$ of source data $\bm{x}_s$:
\begin{equation}
\label{eq:L_s}
\begin{split}
L_s = \bm{E}_{\bm{x}_s \sim D_S} \sum_{n=1}^{N} -I[n = y_s] \log{P_{n}(y|\bm{x}_s)},
\end{split}
\end{equation}
where $\bm{E}$ is the expectation operator. $I[n = y_s]$ is a indicator function, which is 1 when $y_s$ equals $n$ and 0 otherwise; and $P_{n}$ represents  the predicted probability for class $n$. 

\subsubsection{Minimization of the Combined Error}
The minimum combined error $\lambda$ is dependent on the bias between the source domain and the target domain and the capacity of the classifier space $C$. Ben-David \etal suggest that if the minimum combined error $\lambda$ is relatively large, it is not likely to build a hypothesis that performs well in both the source and target domain \citep{ben-david_theory_2010}. $\lambda$ is usually small if the capacity of the classifier space is sufficiently high. In the present paper, the classifiers are neural networks, whose parameters are large enough to form a high-capacity 
classifier space, and thus the $\lambda$ can be assumed to be small.

\subsubsection{Minimization of the $C \Delta C$-distance}
After minimizing the source error $\epsilon_{S_g}(c)$ and the combined error $\lambda$, the remaining items in \eqref{eq:sup_error} is the $C \Delta C$-distance. The $C \Delta C$-distance in \eqref{eq:sup_error} measures the discrepancy between the disagreement of two classifiers ($c_1$ and $c_2$) on the transformed source domain and that in the target domain. Because the feature generators and the classifiers are first optimized by the cross-entropy loss $L_s$ shown in \eqref{eq:L_s} to minimize the source error, the classifier space after optimization may change to the space where the source error is minimum. In the optimized classifier space $C$, the disagreement of two arbitrary classifiers on the source domain will be small (see \autoref{fig:EDH_principle}), which indicates that $\epsilon_{S_g}(c_1, c_2)$ will be small. Therefore, the $
C\Delta C$-distance shows that a small change of the classifier in the source domain can result in a large change of the classifier in the target domain. To minimize the $C \Delta C$-distance, the upper bound of the disagreement between two classifiers in the target domain should be minimized:
\begin{equation}
\label{eq:minimize_CdC_distance}
\begin{split}
&\min_{T_g}\sup_{c_1,c_2 \in C}\bm{E}_{g(\bm{x}) \sim T_g}\big[|c_1(g(\bm{x})) - c_2(g(\bm{x}))|\big].
\end{split}
\end{equation}

The disagreement of classifiers can be measured but its upper bound cannot be measured directly. To estimate the upper bound, \eqref{eq:minimize_CdC_distance} is converted into  a min-max problem. The discrepancy can be changed by two different classifiers $c_1$ and $c_2$, and the transformed target domain $T_g$ can be adjusted by a feature generator $g$. The classifiers $c_1$ and $c_2$ are trained to maximize the discrepancy between each other and find the upper bound of the discrepancy of hypotheses. Then the feature generator $g$ is trained to minimize the discrepancy of the classifiers and decrease the upper bound of the classifier discrepancy. The overall process is shown in \autoref{fig:EDH_principle} and can be concluded by:
\begin{equation}
\label{eq:min-max_problem}
\begin{split}
&\min_{g}\max_{c_1, c_2}\bm{E}_{g(\bm{x}) \sim T_g}\big[|c_1(g(\bm{x})) - c_2(g(\bm{x}))|\big].
\end{split}
\end{equation}

There are only one feature generator and two classifiers in \eqref{eq:min-max_problem}, which cannot be directly used to train the hypotheses in the present paper. The reason is that the present paper combines multiple diverse hypotheses, including multiple feature generators $g_i$ and classifiers $c_k$, to further decrease the target error. After increasing the number of classifiers, the loss of classifier discrepancy $L_{d, c}$ should be:
\begin{equation}
\label{eq:L_dis_C}
\begin{split}
&L_{d,c} = \sum_{i=1}^{n_g}\sum_{j=1}^{n_c}L_{d, c}^{ij},\\
&L_{d,c}^{ij}=\bm{E}_{\bm{x}_t \sim D_T} \Big[{\sum_{n=1}^{N} (\lvert P_{n}^{ij}(y_t|\bm{x}_t) - \bar{P}_{n}^{i}(y_t|\bm{x}_t)\lvert)}/N\Big],\\
&P_{n}^{ij}(y_t|\bm{x}_t) = c_k(g_i(\bm{x}_t))[n], k = n_c(i-1) + j,\\
&\bar{P}_{n}^{i}(y_t|\bm{x}_t) = \bm{E}_{j\in[1,n_c]}\big[P_{n}^{ij}(y_t|\bm{x}_t)\big],
\end{split}
\end{equation}
where $L_{d,c}^{ij}$ denotes the loss of classifier discrepancy for a hypothesis whose generator and classifier are $g_i$ and $c_k$, respectively.
$P_{n}^{ij}(y_t|\bm{x}_t)$ indicates predictions for class $n$ when the input is $\bm{x}_t$ and the hypothesis consists of $g_i$ and $c_k$. $D_T$ denotes the original target domain; and $\bar{P}_{n}^{i}(y_t|\bm{x}_t)$ is the average prediction of different hypotheses whose generator is $g_i$.

\subsubsection{Maximization of the Feature Discrepancy}
Besides the classifier discrepancy, it is also necessary to maximize the feature discrepancy to learn a complete set of hidden features, which is demonstrated in \autoref{theorem:ensemble_learners}. Moreover, both the accuracy of each hypothesis and the diversity of a hypothesis can decrease the target error of the combined hypotheses, which is supported by \autoref{theorem:error_discrepancy_decomposition} and \autoref{fig:EDH_principle}. Therefore, the present paper designs a loss $L_{d,g}$ to maximize the feature discrepancy:
\begin{equation}
\label{eq:L_dis_G}
\begin{split}
&L_{d,g} = \bm{E}_{\bm{x}}\Big[\sum_{i=1}^{n_g}\lvert g_i(\bm{x}) - \bm{E}_{i\in[1,n_g]} g_i(\bm{x})\lvert \Big].
\end{split}
\end{equation}

\subsubsection{Maximization of the Classifier Confidence}
EDH can further decrease the target error of the individual hypotheses by utilizing the cluster assumption. For unsupervised learning, labels of data are unknown, but one can still analyze the unlabeled dataset and reveal the intrinsic properties of the dataset, by clustering the data. Similarly, the cluster assumption can also improve the unsupervised domain adaptation. The cluster assumption presents that the dataset can be divided into different clusters and the data that come from the same cluster should belong to the same class \citep{shu_dirt-t_2018}. Therefore, the decision boundary should not cross high-density areas. This objective can be realized by minimizing the conditional entropy loss \citep{grandvalet_semi-supervised_2004}:

\begin{equation}
\label{eq:L_ent}
\begin{split}
&L_e = -\sum_{i=1}^{n_g}\sum_{j=1}^{n_c}\bm{E}_{\bm{x}_t \sim D_T}[P^{ij}(y_t|\bm{x}_t)^T \ln P^{ij}(y_t|\bm{x}_t)],\\
&P^{ij}(y_t|\bm{x}_t) = c_k(g_i(\bm{x}_t)), k = n_c(i-1) + j.\\
\end{split}
\end{equation}

There are several advantages to minimizing conditional entropy. First, the hypothesis space can be further constrained and the optimized hypothesis space may not include the region where the target data are dense. Therefore, the $C \Delta C$-distance $d_{C \Delta C}(S_g, T_g)$ can be decreased because changing the decision boundary in low-density regions will not cause a large classifier discrepancy in the target domain. Then the upper bound of the target error will reduce because of the decrease of the $C \Delta C$-distance. Second, the decision boundary can utilize the density information of the target data to 
restrict its choice, which may provide a good prior probability distribution and increase the classification accuracy.

\begin{figure*}[h!]
    \centering
    \includegraphics[width=\textwidth]{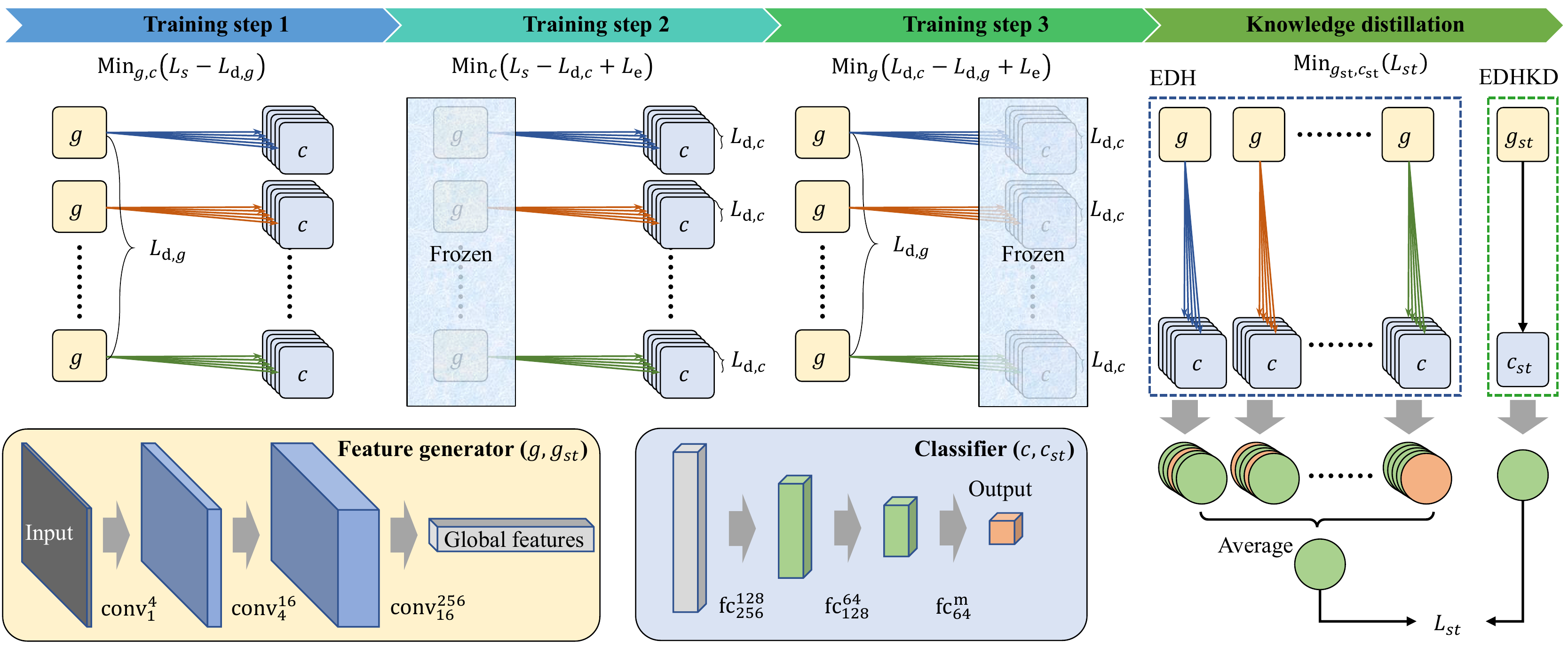}
    \caption{The network architecture and the training steps of the proposed EDH and EDHKD. $g$ and $c$ represent a feature generator and a classifier, respectively. The superscript and the subscript of the convolutional layer indicate the channel of output features and that of input features, respectively. Training steps and loss functions are discussed in \autoref{subsec:theoretical_basis} and \autoref{subsec:network_architecture}.} 
    \label{fig:network_architecture}
\end{figure*}

\subsection{Network architecture and training steps of EDHKD}
\label{subsec:network_architecture}
\subsubsection{Lightweight network architecture}
There is usually a trade-off between the fitting ability and the generalization ability, for a neural network. A large network has more parameters and can learn deeper features but it may overfit the training dataset, especially when there is no target validation set to determine the early-stop time. A lightweight neural network may have a better generalization ability but it may not fit the training dataset as accurately as with a large network. The present paper addresses this issue by adopting the ensemble method. The ensemble method is attractive because it can boost weak learners to make accurate predictions \citep{zhou_ensemble_2012}. Even though each learner is weak, the generalization ability of the ensemble method can be better than with a strong learner. Moreover, after the knowledge distillation, even a single learner can learn the knowledge from the ensemble learner.
This property allows us to design a lightweight neural network (see \autoref{fig:network_architecture}) so as to avoid overfitting, which is observed using the large network that was designed in our previous work \citep{zhang_unsupervised_2020}, and increase the efficiency. 


The feature generator in the present paper is a convolutional neural network. The kernel size and the stride in the first two convolutional layers are $(1\times1)$ and 1, respectively. These two layers are similar to a multi-layer perceptron which can extract deeper features. The kernel size in the third layer equals the size of the input image. This layer convolves all features into a global feature vector. After each convolutional layer and before the nonlinear activation function (ReLu6), batch normalization is applied to normalize the magnitude of the signals from different sensors and channels. Because extracting shallow features has already reduced the dimension of the input image, the max-pooling layer is not adopted here.

The classifier in the present paper is an artificial neural network that consists of three fully connected layers. The first two fully-connected layers with batch normalization and a nonlinear activation function (ReLu6) map the global feature vector to 64 hidden features. The last fully-connected layer maps these hidden features to the classification scores of the input image. 

There are five feature generators and 25 classifiers in the EDH. Each generator is shared by five classifiers to achieve the min-max training objective. Five generators are utilized to map the input data to different distributions. The classification results of each hypothesis are voted to find the most possible class of human intent, which is the class with the highest probability. The average predictions of the EDH are regarded as pseudo labels and are utilized to train the student learner (EDHKD), which only includes one feature generator and one classifier and is efficient.

\subsubsection{Training steps}
\label{subsubsec:training_steps}
The loss functions and their theoretical bases have been introduced in \autoref{subsec:theoretical_basis}. The present section indicates how to combine these functions and train the networks.

\textbf{Step 1.} 
Based on the theoretical analysis in \autoref{subsec:theoretical_basis}, the feature generators and classifiers are first trained to minimize the source classification error. The feature generators are also optimized to maximize the feature discrepancy: 
\begin{equation}
\label{eq:step_1}
\begin{split}
\min_{g_i, c_k} L_s - \xi_g L_{d,g}, i \in [1,n_g], k \in [1,n_g\cdot n_c],
\end{split}
\end{equation}
where $\xi_g$ is a constant weight of loss function $L_{d,g}$.

\textbf{Step 2.} 
To find the upper bound of the classifier discrepancy, the classifier discrepancy should be maximized to maximize the $C \Delta C$-distance. Besides, the entropy loss should be minimized to utilize the cluster assumption and increase the classifier confidence. In this phase, the classifiers $c_k$ are trained using the source error loss, the classifier discrepancy loss, and the entropy loss while the feature generators $g_i$ are frozen:
\begin{equation}
\label{eq:step_2}
\begin{split}
\min_{c_k} L_s - \xi_c L_{d, c} + \xi_e L_e, k \in [1,n_g\cdot n_c],
\end{split}
\end{equation}
where $\xi_c$ and $\xi_e$ are the constant weight of loss function $L_{d,c}$ and $L_e$.

\textbf{Step 3.} 
The objective of maximizing the classifiers is to find the upper bound of the classifier discrepancy. After finding the upper bound of the classifier discrepancy, the feature generators should be trained to minimize the upper bound of the classifier discrepancy and thus decrease the $C \Delta C$-distance. Moreover, the feature discrepancy is maximized to reduce the error of the average hypothesis. The entropy loss should also be minimized to transform the features of different classes into different clusters. In this step, classifiers $c_k$ are frozen and the feature generators $g_i$ are trained to adjust the target domain to match the source domain and minimize the upper bound of the classifier discrepancy, maximize the feature discrepancy, and increase the classification confidence: 
\begin{equation}
\label{eq:step_3}
\begin{split}
\min_{g_i} L_{d, c} - L_{d,g} + \xi_e L_e, i \in [1,n_g].
\end{split}
\end{equation}

These three steps are iterated for many mini-batches until achieving a fixed epoch. 

\textcolor{black}{\textbf{Knowledge distillation.}} After training, the EDH model can provide the pseudo labels of each target data. These pseudo labels are used to train the student network EDHKD by minimizing the cross-entropy between the pseudo labels and the prediction of EDHKD
\begin{equation}
\label{eq:L_student}
\begin{split}
L_{st} = \bm{E}_{\bm{x}_t \sim D_T} \sum_{n=1}^{N} -P_{n}(y_{te}|\bm{x}_t) \log{P_{n}(y_{st}|\bm{x}_t)},
\end{split}
\end{equation}
where $y_{te}$ and $y_{st}$ are the prediction of $\bm{x}_t$ by the teacher (EDH) and the student (EDHKD), respectively.

\subsection{Experimental setup}
Similar to previous research \citep{saito_maximum_2018}, we first implemented a pre-experiment to classify 2D intertwining moon points. The source samples consist of an upper moon and a lower moon and are shown in \autoref{fig:learning_process_moon}. Two moons are labeled 0 and 1, and each moon contains 1500 points with Gaussian noise $e\sim N(0, 0.05)$. The target samples are generated by rotating and translating the source samples. The rotation angles and the translations are shown in \autoref{fig:learning_process_moon}. The decision boundary of the binary classifier was drawn to visualize the classification results, which can be seen in  \autoref{fig:learning_process_moon}. Because the moon dataset has just two dimensions, the network size was further decreased. The network architecture and the hyper-parameters used in this experiment are given in \autoref{tab:training_parameters}.

\begin{table}[!h]
\centering
\caption {\label{tab:training_parameters} Hyperparameters of EDH and EDHKD on different datasets.}
\renewcommand{\arraystretch}{1} 
\begin{center}
\resizebox{0.7\columnwidth}{!}{
\begin{tabular}{l c c}
\toprule
Dataset & Moon & ENABL3S and DSADS\\
\midrule
Optimizer & Adam & Adam\\
Learning rate & $1\times10^{-3}$ & $2\times10^{-4}$\\
Weight of feature discrepancy loss $\xi_g$ & 3 & 5\\
Weight of classifier discrepancy loss $\xi_c$ & 3 & 5\\
Weight of entropy loss $\xi_e$ & 1 & 0.01\\
Training epochs & 50 & 100\\
Batch size & 200 & 256\\
Number of generators & 5 & 5\\
Number of classifiers & 25 & 25\\
Generator update number & 3 & 4\\
Architecture of generator & ANN & CNN\\
Length of global feature & 32 & 256\\
\end{tabular}}
\end{center}
\end{table}

In the present paper, two public datasets are utilized to evaluate the performance of the proposed methods. One dataset is the encyclopedia of able-bodied bilateral lower limb locomotor signals (\href{https://doi.org/10.6084/m9.figshare.5362627}{ENABL3S}) as collected by the Northwestern University \citep{hu_benchmark_2018}. The other dataset is the daily and sports activities data set (\href{http://archive.ics.uci.edu/ml/datasets/Daily+and+Sports+Activities}{DSADS}) provided by Bilken University \citep{barshan_recognizing_2014}. The signals in these two datasets have already been filtered and segmented by their developers, and the detailed processing methods are found in their papers \citep{hu_benchmark_2018, barshan_recognizing_2014}.

For ENABL3S \citep{hu_benchmark_2018}, the signals of ten able-bodied subjects were captured from experiments. Subjects were asked to walk on different types of terrains, and the locomotion mode was switched between standing (St), level ground walking (LW), stair ascent (SA), stair descent (SD), ramp ascent (RA), and ramp descent (RD). ENABL3S \citep{hu_benchmark_2018} contains filtered EMG, IMU, joint angle signals. The filtered signals were segmented by a 300 ms wide sliding window. Human intent is not an intuitive signal and its time of occurrence is difficult to accurately define. Therefore, gait events, such as toe-off and heel-contact, were utilized by previous researchers to estimate the time of switching between different locomotion modes \citep{huang_continuous_2011, xu_real-time_2018, zhang_subvision_2021}. The sliding windows began 300 ms before the gait events, and thus the segmented signals can be utilized to predict locomotion intent.

The shallow features were then extracted from the segmented signals. Ten shallow features of EMG signals are the mean absolute value, waveform length, the number of changing slope signs, the number of zero-crossing points, and the coefficients of a sixth-order autoregressive model. Six features, the mean, the standard deviation, the maximum, the minimum, the initial value, and the final value, were extracted from each channel of IMU and joint angle signals. The features of EMG and IMU are reshaped to a $33 \times 12$ single-channel image, in the present paper.

DSADS \citep{barshan_recognizing_2014} invited eight able-bodied subjects to perform 19 different activities, including sitting, jumping, standing, riding a bike, and running. DSADS \citep{barshan_recognizing_2014} contains signals captured from five 9-axis IMUs. The captured signals were segmented by a 5 s window. The features extracted from the segmented signals are the same as those of IMU signals in ENABL3S. Extracted features are reshaped to a $45 \times 6$ single-channel image. There is no transition between different motion modes in DSADS and thus DSADS cannot be used to predict human intent. DSADS is utilized to classify the human locomotion modes and compare the EDH with a benchmark method that has been proposed previously\citep{zhang_unsupervised_2020}.

ENABL3S and DSADS datasets include 22,000 and 9,000 signal segments, respectively. The data from each subject were randomly shuffled and divided into a training set (70\%) and a test set (30\%). It should be noted that there is no validation set because we assume the target dataset is unlabeled and we cannot obtain a labeled target validation set to fine-tune the neural network nor determine the early-stop time. In every experiment, a target subject was selected and the remaining subjects were source subjects, which may be called a leave-one-subject-out test. The proposed EDH was trained using the labeled source training dataset and the unlabeled target training dataset. After training a fixed number of epochs, the EDH was tested using the test set to evaluate the performance of cross-subject adaptation. After finishing an experiment, the next experiment selected a different subject as the target subject until transversing all subjects. 

In experiments, the proposed EDH and EDHKD were compared with different state-of-the-art methods, including LDA (linear discriminate analysis), SVM (support vector machines), ANN (artificial neural networks), CNN (convolutional neural network), the BM (benchmark) method utilized in the previous study \citep{zhang_unsupervised_2020}, DANN \citep{ganin_domain-adversarial_2016}, MCD \citep{saito_maximum_2018}, MMD (maximum mean discrepancy) \citep{long_learning_2017}, CORAL (correlation alignment) \citep{sun_deep_2016}, and DFA (Discriminative feature alignment) \citep{wang_discriminative_2021}. LDA, SVM, ANN, and CNN were only trained using the source training datasets and tested using the target test data. These four methods can be considered baseline methods. The network architectures of the feature generator and the label classifier used in CNN, DANN, MCD, MMD, CORAL, DFA, EDH, and EDHKD are the same. There is a domain classifier in DANN. The BM method used in the previous study adopted the same training strategy and loss functions as in MCD, but the network size of the BM method was much larger. The network architecture of the BM method is described in detail in the previous paper \citep{zhang_unsupervised_2020}.

All experiments were implemented on a desktop computer with an Intel Core i7-6700 CPU and an NVIDIA GeForce GTX 1080 GPU. 

\section{\textcolor{black}{Results}}
\label{sec:Cross-subjectResults}
\subsection{Evaluation using moon dataset}
The EDHKD, which only includes one feature generator and one classifier, was found to learn the knowledge from EDH successfully and outperform other methods on all moon datasets with different rotation angles and translations. As shown in \autoref{fig:learning_process_moon}, the proposed EDHKD can find the decision boundary that classifies the target moon dataset more accurately.  After training 50 epochs, the proposed EDHKD can achieve $96.9\%\pm5.0\%$ classification accuracy on the target moon dataset, which is significantly higher than ANN ($\uparrow 13.7\%,
P<0.001$) and DANN ($\uparrow 12.5\%, P<0.001$). $P$ is the probability that the null hypothesis is true, which is calculated using a one-way ANOVA with a post hoc test. The difference is regarded as significant if $P<0.05$. The EDHKD achieves $5.2\%$ higher classification accuracy than MCD, but the difference is not significant ($P=0.06$). The standard deviation of the classification accuracy for EDHKD is at least 1\% lower than those for the other three methods, which indicates that EDHKD is more stable than the other three methods. 

\begin{figure*}[htpb]
    \centering
    \includegraphics[width=\linewidth]{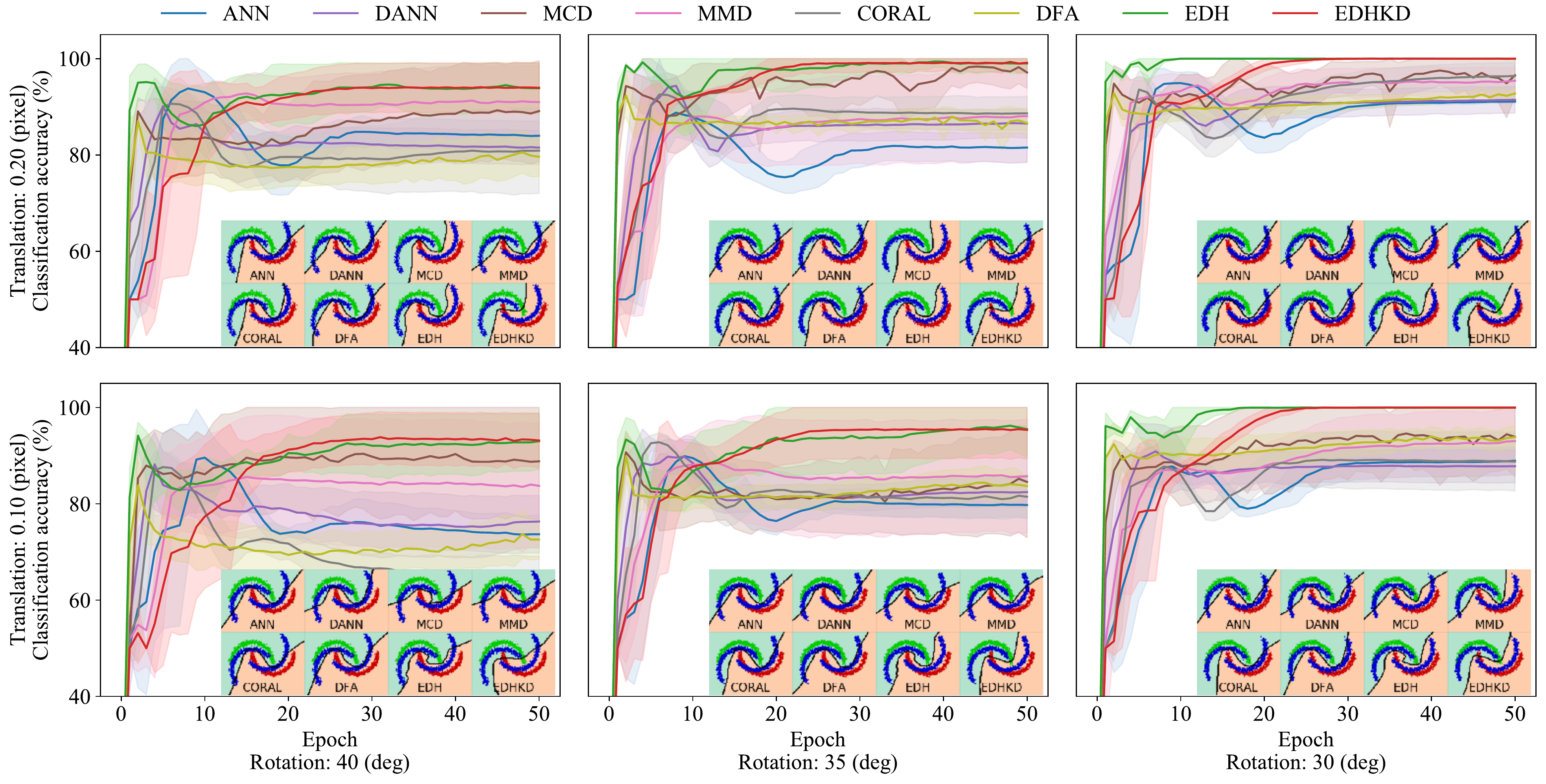}
    \caption{Decision boundary and learning curves of different methods for classifying target test set of the moon dataset in the training process. Green and red points represent the source data of two different classes. Blue points are the target data that should belong to two categories but are not labeled. The line and the shaded region represent the mean and the standard deviation, respectively, of the average classification accuracy in five trials. 
    }
    \label{fig:learning_process_moon}
\end{figure*}

The learning curves of the proposed EDHKD are higher than those of the other methods. All networks enter a plateau after training about 30 epochs. The plateau of the proposed EDHKD is higher and more stable than in MCD. The possible reason may be that the diverse generators and classifiers of the EDH decrease the variance of the classification accuracy, and then the EDHKD can learn the knowledge from the accurate pseudo labels of the EDH. Before entering the plateau, all networks except EDHKD meet a peak accuracy but the accuracy decreases after further training. This phenomenon reflects the stability of the EDHKD and the importance of the validation set, which helps to determine the early-stop time. Without the validation set, the performance of the network after training fixed-number epochs may be much worse than the peak performance. The proposed EDHKD mitigates this problem and achieves a higher steadiness, as shown in \autoref{fig:learning_process_moon}.

\subsection{Evaluation on ENABL3S}
On ENABL3S \citep{hu_benchmark_2018}, the proposed EDH was able to predict the locomotion intent of target subjects more accurately (mean over 95\%) and stably (standard deviation (std) less than 2\%) than all other methods (see \autoref{fig:NW_learning_process} and \autoref{tab:classification_nw}). However, the forward time of EDH was as long as 6.5 ms, which may cause the time delay of recognition results.
After knowledge distillation, the EDHKD could decrease the forward time to 1 ms. The classification accuracy of the EDHKD decreased to 94.4\% but the decrease was not significant ($P=0.4$). Compared to the BM method \citep{zhang_unsupervised_2020}, MCD utilizes the same loss functions but different network architecture and achieves 0.9\% higher accuracy. 
Although the improvement is not significant ($P=0.4$), it shows the advantage of utilizing a suitable network architecture for classifying human locomotion intent. Using the same generator and classifier architecture, the proposed EDH further improves the classification accuracy, which is 1.1\% higher than with MCD but the difference is not significant ($P=0.17$), either. Moreover, the accuracy of EDH is 2.0\% higher than that of the BM method and the improvement is significant ($P\approx0.05$). Although EDHKD only utilizes one feature generator and one classifier, it still achieves 0.5\% and 1.3\% higher classification accuracy than MCD and the BM method, but the difference is not significant ($P>0.19$). Therefore, both optimizing the network architecture and combining diverse hypotheses contribute to the improvement of the accuracy of predicting the locomotion intent of target subjects. The knowledge distillation could increase the efficiency of the network and simultaneously remain the high classification accuracy, which is beneficial for real-time recognition.

The proposed EDH and EDHKD perform significantly ($\geq 6.1\%, P\leq0.003$) better than non-adapted methods, like LDA, SVM, ANN, and CNN (see \autoref{tab:classification_nw}) in predicting the locomotion intent of the target subjects, which indicates that the proposed method can successfully transfer the learned knowledge from the source domain to the target domain. The target classification accuracy using EDHKD and MCD are significantly ($\geq 5.4\%, P\leq0.003$) higher than when using DANN, which indicates that the domain classifier cannot align the features of each class in two domains and may not work well if the domain biases are not uniform in different classes. Moreover, the EDHKD and MCD also achieve higher classification accuracy ($\geq 1.2\%, P\leq0.18$) than the non-adversarial domain adaptation method, e.g., MMD and CORAL, which shows that aligning the features in each class is more accurate than aligning features globally.

The source classification accuracy is always higher than the target accuracy, which is reasonable because the source data are labeled. The source classification accuracy using the deep neural networks (e.g., CNN, DANN, MCD, and EDH) are significantly higher ($\geq 1.4\%, P\leq0.01$) than that using shallow classifiers (LDA, SVM, and ANN), which validates the feasibility of the designed convolutional neural network designed in the present paper. Because the size of the proposed convolutional neural network has been decreased and the convolutional neural network can utilize GPU to compute in parallel, the forward time of CNN, DANN, MCD, MMD, CORAL, DFA, and EDHKD are similar to that of shallow classifiers. 

\begin{figure*}[htpb]
    \centering
    \includegraphics[width=\linewidth]{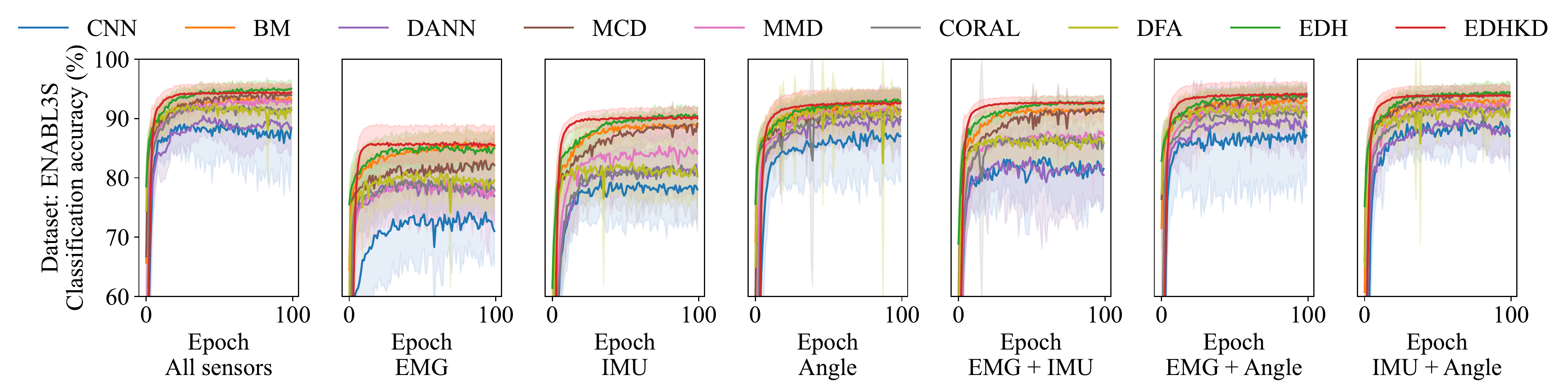}
    \caption{Learning curves of different methods for classifying target test set of the ENABL3S dataset in the training process. The line and the shaded region represent the mean and the standard deviation, respectively, of the average classification accuracy of 10 different leave-one-subject-out experiments. 
    }
    \label{fig:NW_learning_process} 
\end{figure*}

Different sensors affect the accuracy of predicting the locomotion intent of target subjects. As shown in \autoref{fig:NW_learning_process}, the classification accuracy of using all sensors is higher than when using part of the sensors. Although EMG signals seem the least accurate to classify locomotion modes, they can still improve the classification accuracy after being fused with other sensors (e.g., IMU and angle sensors). It is worth noting that the proposed EDHKD performs better than other methods for all different combinations of sensors, which further validates the performance of the proposed EDHKD. 

The best accuracy is not ensured without using the validation set to terminate the training early, but the overfitting problem seems negligible on ENABL3S. The proposed EDH tends to converge after training 50 epochs and its variance is small after converging. Without ensemble diverse hypotheses, the other methods are less stable during training and there is a risk to achieve a worse accuracy after training a fixed number of epochs.

\begin{table}[htbp]
\centering
\caption {\label{tab:classification_nw} Accuracy and forward time of classifying the locomotion modes for the source and target subject of ENABL3S and DSADS using different methods. Std denotes the standard deviation. \textcolor{black}{The forward time indicates the online executing time of classifying each signal segment.}}
\begin{center}
\resizebox{\textwidth}{!}{%
\begin{tabular}{l c c c c c c c c c c}
\toprule
Dataset & \multicolumn{5}{c}{ENABL3S} & \multicolumn{5}{c}{DSADS}\\
\midrule
Domain & \multicolumn{2}{c}{Source} & \multicolumn{2}{c}{Target} & & \multicolumn{2}{c}{Source} & \multicolumn{2}{c}{Target} &\\
\midrule
Methods & Mean(\%) & Std(\%) & Mean(\%) & Std(\%) & \textcolor{black}{Time(ms)} & Mean(\%) & Std(\%) & Mean(\%) & Std(\%) & \textcolor{black}{Time(ms)}\\
\midrule
LDA       &        92.5	&\textbf{0.2} &        85.5	&       3.9 &\textbf{0.2}   &        98.2	&\textbf{0.1} &        92.4	&\textbf{3.3}&\textbf{0.2}\\
SVM       &        89.5	&        1.3  &        78.4	&        6.8        & \textbf{0.2}   &        96.7	&        1.0  &        81.6	&        6.6 &        \textbf{0.2} \\
ANN       &        93.6	&        0.4  &        87.4	&        3.4        & 0.3   &        98.0	&        0.5  &        83.9	&        9.5 &0.3\\
CNN       &\textbf{96.7}	&        0.3  &    88.3	&        5.4        & 1.0   &\textbf{99.3}	&        0.4  &        90.0	&        6.1 &        1.0\\
BM \citep{zhang_unsupervised_2020} &94.5	&2.3  &93.1	&2.6            & 4.0   &99.2	&        0.3  &        90.3	&        3.7 &        3.5\\
DANN \citep{ganin_domain-adversarial_2016}  &96.2	&0.3  &88.5	&4.8    & 1.0   &        99.2	&        0.2  &        91.1	&        5.2 &        1.0\\
MCD \citep{saito_maximum_2018}      &95.0	&1.2  &93.9	&1.8            & 1.0   &        98.8	&        0.5  &        95.3	&        4.5 &        1.0\\
MMD \citep{long_learning_2017}        &96.1	&0.6	&92.7	&2.2   &1.0     & \textbf{99.3}	&0.2	&95.4	&3.3	 &1.0\\
CORAL \citep{sun_deep_2016}           &96.3	&0.4	&91.4	&3.4   &1.0     & \textbf{99.3}	&0.2	&91.7	&4.3	 &1.0\\
DFA \citep{wang_discriminative_2021}  &94.8	&1.7	&91.8	&3.0   &1.0     & 98.8	&0.7	&92.5	&3.3     &1.0\\
\textbf{EDH (Ours)}       &96.0	&        0.7  &\textbf{95.1}  &\textbf{1.7}     & 6.5   &99.1	&0.4  &\textbf{97.5}  &        4.7 &        5.2\\
\textbf{EDHKD (Ours)}     &$\backslash$ 	&$\backslash$ & 94.4  &\textbf{1.7}     & 1.0   &     $\backslash$ 	&  $\backslash$       & 97.4  & 4.9 &        1.0\\
\end{tabular}
}
\end{center}
\end{table}

\subsection{Evaluation using DSADS}
The performances of the unsupervised domain adaptation method was also evaluated using DSADS, to classify more human activities, including running, riding bicycles, and playing basketball. As shown in \autoref{fig:UCI_learning_process} and \autoref{tab:classification_nw}, the proposed EDHKD outperforms the other methods with a high (mean = 97.4\%) and consistent target classification accuracy (std = 4.9\%), which is significantly higher ($\uparrow 7.1\%, P=0.012$) than the BM method adopted in the previous paper \citep{zhang_unsupervised_2020}. 

It is worth noting that the MCD, which utilizes the same loss functions as the BM method, also achieves 5\% higher target classification accuracy than the BM method and the improvement is significant ($P=0.04$). This result shows that optimizing the network architecture is necessary to avoid overfitting when the dataset is small. As shown in \autoref{fig:UCI_learning_process}, the BM method first achieves the highest accuracy after training about 20 epochs and then starts to overfit the training set. After training 100 epochs, the final accuracy for the BM method is much lower than the peak accuracy. After decreasing the size of the neural networks, both MCD and EDHKD become more stable during training and enter a plateau after training 50 epochs. However, the learning curves of MCD are still less stable than that of EDHKD, which reveals that EDHKD increases the steadiness of the classification.  

\begin{figure*}[htpb]
    \centering
    \includegraphics[width=\linewidth]{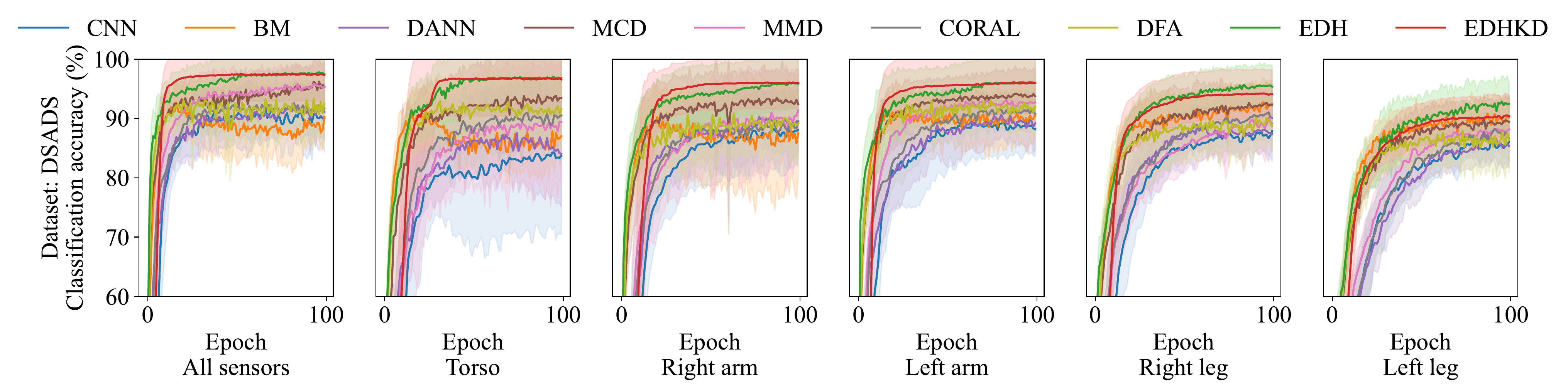}
    \caption{Learning curves of different methods for classifying target test set of DSADS dataset in the training process. The line and the shaded region represent the mean and the standard deviation, respectively, of the average classification accuracy of 8 different leave-one-subject-out experiments. 
    }
    \label{fig:UCI_learning_process} 
\end{figure*}

For DSADS, the sensor position does not severely affect the target classification accuracy. Five IMUs were worn on the different body parts of the subjects, including torso (T), right arm (RA), left arm (LA), right leg (RL), and left leg (LL). The proposed EDHKD using the single sensor performs works as well as using all sensors except using the IMU on the left leg. Whichever sensor is utilized, EDHKD outperforms all other methods. 

A surprising result is that LDA can still achieve 92.4\% accuracy for classifying the locomotion modes of target subjects, which is even higher than using CNN but the difference is not significant ($P=0.25$). LDA is an analytical method and thus is very efficient for real-time computing (forward time = 0.2 ms). Sometimes both target label and target data are not available and only source data are available to train the classifier. In these situations, LDA seems still a simple but powerful method to find the decision boundary in a low manifold space and ensure an acceptable generalization ability. If the target data are available, using EDHKD is better because it achieves a significantly higher ($\uparrow5.0\%, P = 0.01$) accuracy than with LDA but only costs 1.0 ms, which is still efficient.

\subsection{Learned diverse features}
To further validate whether the proposed EDH learns diverse features, the present paper visualizes t-SNE \citep{van_der_maaten_visualizing_2008} projection of the non-adapted features and the hidden features adapted by different generators. As shown in \autoref{fig:tsne}, all generators can improve the alignment between the source features and the target features. Without adaptation, the source features distribute differently from the target features. After adaptation, most source and target features belonging to the same class are located in the same region. Moreover, different generators transform the input feature into different feature spaces and the projected feature distributions for different generators are different from each other.

\begin{figure*}[h!]
    \centering
    \includegraphics[width=\linewidth]{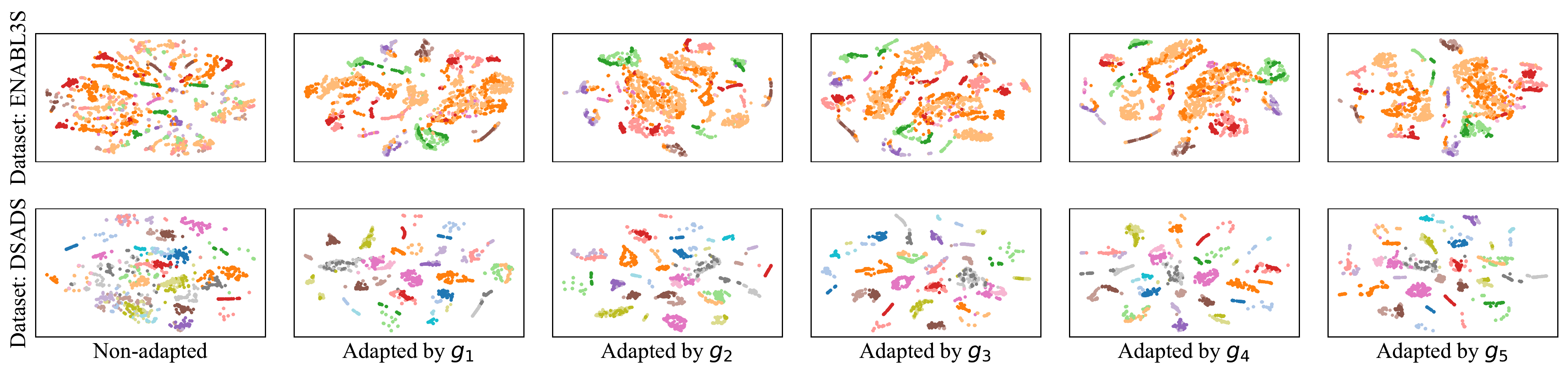}
    \caption{The visualization of t-SNE projection of non-adapted features and the adapted features extracted from five feature generators of EDH. The light-color points and the dark-color points represent the source and the target data respectively. The points of the same color belong to the same class. All features are extracted from the training set. After adaptation, features of the same class concentrate on a similar region, and the clearance between different clusters are larger than that before adaptation.}
    \label{fig:tsne} 
\end{figure*}
\section{Discussion}
\label{sec:Cross-subjectDiscussion}
The present paper proposed the ensemble diverse hypotheses and knowledge distillation (EDHKD) to classify the locomotion intent and the modes of target subjects without labeling their data. Compared to the previous research \citep{zhang_unsupervised_2020}, the present paper aimed to address the issue that the single generator may only learn a subset of the features and the large-size network may overfit the source data if there is no validation set of the target data. The proposed EDHKD realized the mentioned objectives by combining diverse feature generators and classifiers, optimizing loss functions, and lightening the base learner. The proposed EDHKD was evaluated using three different datasets and compared with different methods, including shallow classifiers, deep classifiers, and also the BM method used in the previous research \citep{zhang_unsupervised_2020}. Experimental results indicated that EDHKD achieved 96.9\%, 94.4\%, and 97.4\% average classification accuracy on the target dataset of the moon, ENABL3S, and DSADS datasets, which were higher than all other methods. Specifically, EDHKD achieved 1.3\% and 7.1\% higher accuracy for classifying the target data of ENABL3S and DSADS than in the BM method. These results have validated the feasibility of the proposed EDHKD method. Moreover, the EDHKD remained efficient (forward time $leq 1$ ms) as it only included one feature generator and one classifier and distilled the knowledge from the teach network (EDH), which validated the hypothesis that the light student network could learn the knowledge from the pseudo labels of a large teach network.

There are a variety of underlying reasons for the achieved improvements, which are discussed in the following paragraphs.

First, using two intelligent agents is better than one. As shown in \autoref{fig:learning_process_moon}, \autoref{fig:NW_learning_process}, and \autoref{fig:UCI_learning_process}, the proposed EDH achieved 5.2\%, 1.1\%, and 2.2\% higher target classification accuracy than the MCD even if the network architecture of their feature generators and classifiers are the same. 
Minimizing the loss functions of MCD could decrease the upper bound of target error theoretically. However, MCD
only utilizes one feature generator and may only learn a subset of the hidden features. Besides, there is no guarantee to find the global minima of MCD loss functions by training the feature generators and classifiers through backpropagation. One reason is that the parameter number of a neural network is large and the loss function may not be convex, which may cause the neural network to be trapped in the local minima. Another reason is the limitation of multi-objective optimization. There are several items in equation \eqref{eq:sup_error}, including the source error and the $C\Delta C$ distance of the source domain and the target domain. These two objectives may conflict with each other. Optimizing multiple loss functions may achieve a balance between them rather than finding the global minima for each loss function. Therefore, the MCD may still make mistakes after training. EDH cannot fully solve this problem but can mitigate the problem. Each hypothesis may make a mistake but different hypotheses seldom make the same mistake at the same time. Voting the results of all hypotheses may further reduce the classification error. Besides, the diverse-feature loss drives the different base learners to learn different parameters, which avoids causing the network to be trapped in the same local minima.

There are about $1.6\times10^6$ and $4.2\times10^4$ trainable parameters in the feature generator and the classifier of the EDHKD. The classifier of the BM method has $3.4\times10^4$ trainable parameters, which is comparable to that of EDHKD. However, the feature generator of the BM method has $1.3\times10^7$ trainable parameters, which is much larger than that of EDHKD. The higher parameter number usually indicates a higher probability of overfitting because of a higher fitting ability. The overfitting did happen while training the BM method with the DSADS dataset. As shown in \autoref{fig:UCI_learning_process}, the BM method achieved the peak accuracy after training only 20 epochs, and then it started to overfit. Because there was no validation set to determine the early stop time, the training epoch number was fixed. The final classification accuracy of the BM method was much lower than its peak accuracy. After lightening the network, the learning curves of the MCD, EDH, and EDHKD became more stable (see \autoref{fig:UCI_learning_process}). Because the loss functions of MCD and the BM method are the same, the only reason for the more stable learning curves is the reduction of the size of the feature generator. It should be noted that the BM method seems not to overfit on the ENABL3S dataset, which indicates that either increasing the dataset size or decreasing the neural network size could stabilize the learning curve. In reality, a large data set may be hard to obtain, and selecting a suitable network size is still important.

It is hard to speculate the category of the target data that are not labeled. Therefore, some inner properties of the data should be utilized to speculate their categories. Clustering is a common technique for unsupervised learning, and it aims to find data clusters where the data in the same cluster are more similar to each other than to those in other clusters. It is intuitive to make use of the cluster property for unsupervised cross-subject adaptation. The extracted features validate the cluster assumption. As shown in \autoref{fig:tsne}, the points of the same class (same color) are closer to each other than to those of different classes (different colors). There are distinct clearances between points of different classes. Therefore, it is reasonable to utilize an entropy loss function to promote the decision boundary to move away from the high-density areas. The points of different classes may mix in the initial training phase. In this case, the loss entropy function may have a negative effect. However, there are multiple loss functions for the feature generators of the proposed EDH method, including minimizing the source classification errors and minimizing the classifier discrepancy. The weight of entropy loss is smaller than the above two loss functions. Therefore, the above two loss functions will dominate the training of the neural network in the beginning phase. 
Once the classifier discrepancy and the source errors are minimized, the features will be transformed into different clusters. As shown in \autoref{fig:tsne}, the clearances between different feature clusters increase after adaptation. After transforming the features into different clusters, the entropy loss will play a more important role in training and will further decrease the target classification error.

There are some limitations in the present study. First, the proposed EDH did not directly recognize source signals. To decrease the dimension of the input data and increase the efficiency of the algorithm, some shallow features, such as the mean and the standard deviation, were first extracted from the source signals. These features did reduce the dimension of data required to process, but they may lose some important information. For unsupervised cross-subject adaptation, the deep convolutional neural network may be also helpful to extract more user-independent features. Second, unsupervised cross-subject adaptation does not require labeling the target data, but it still requires capturing the target data. In some situations, neither target label nor target data is available, which makes the situation much worse. In these situations, what we can do is training a stable classifier in the source domain. As shown in \autoref{fig:NW_learning_process}, \autoref{fig:UCI_learning_process}, and \autoref{tab:classification_nw}, the CNN is only trained using the source data but still achieved 88.3\% and 90.0\% classification accuracy on target data. However, the generalization ability of CNN is still unsatisfactory, and its performance on the target data is much lower and less stable than with unsupervised domain adaptation methods (e.g., EDHKD and MCD). This limitation is caused by the bias between the source domain and the target domain, and we can do little if neither the target data nor the target label is available. For this extreme situation, it may be better to capture a few samples of target data and design data augmentation methods, few-shot learning methods, and semi-supervised learning to further improve the classification accuracy on the target domain.


\section{Conclusions}
\label{sec:Cross-subjectConclusion}
Although supervised learning can accurately classify the locomotion intent, it requires collecting and labeling a large number of human signals to well fit the pattern of each new subject, which is often infeasible for both subjects and researchers. The present paper proposed an unsupervised cross-subject adaptation method called ensemble diverse hypotheses and knowledge distillation (EDHKD), which did not require the label of the target data. After training the EDHKD with the labeled source data and unlabeled target data, it could accurately classify the locomotion intent and the modes of the target subjects. Compared to previous methods, the proposed EDHKD could learn diverse features and addressed the overfitting issue by training diverse feature generators and classifiers, lightening the base learner, and incorporating the cluster property of the target data. The performance of the EDHKD was theoretically demonstrated and experimentally evaluated using a 2D moon dataset and two public datasets of human locomotion. Experimental results showed that the EDH could achieve 96.9\%, 94.4\%, and 97.4\% average accuracy for classifying target data of the moon dataset, ENABL3S, and DSADS, respectively. Compared to the BM method utilized in the previous work, the proposed EDH increased 1.3\% and 7.1\% accuracy, stabilized its learning curves in the training process, and decrease the forward time to 1 ms. These results validated the effectiveness of the proposed EDHKD for unsupervised cross-subject adaptation. The proposed EDHKD could assist wearable robots to predict the locomotion intent of the target subjects without labeling the data, which will increase the intelligence of the wearable robot and improve the human-robot interaction.

\section{Acknowledgement}
This work was supported by the National Key R\&D Program of China [Grant 2018YFB1305400]; National Natural Science Foundation of China [Grant U1913205, 62103180, and 51805237]; Guangdong Basic and Applied Basic Research Foundation [Grant 2020B1515120098]; the Science, Technology, and Innovation Commission of Shenzhen Municipality [Grant SGLH2018-0619172011638 and ZDSYS-20200811143601004]; the China Postdoctoral Science Foundation
(2021M701577); and Centers for Mechanical Engineering Research and Education at MIT and SUSTech.
\bibliographystyle{elsarticle-num}
\bibliography{references}
\end{document}


\renewcommand*{\figureautorefname}{Fig.}
\captionsetup[figure]{name={Fig.}}

\section{Thought experiment}
As introduced in \citep{allen-zhu_towards_2021}, the real data are usually "multi-view", which indicates that the data may contains multiple hidden features that can be utilized to recognize. \autoref{theorem:single_learner} displays that the single learner may only learn a subset of the hidden features, which degrades the generalization ability of the learner. 
\begin{theorem}
\label{theorem:single_learner}
(Single learner)
Let $g$, $c_1$, and $c_2$ be a feature generator and two linear binary classifiers. Let $\bm{x}_s$ and $\bm{x}_t$ be the source data and target data, respectively. The label of $\bm{x}_s$ and $\bm{x}_t$ are $y_s \in {0, 1}$ and $y_t \in \{0, 1\}$. 
Assume that three orthogonal hidden features $\bm{v}_1$, $\bm{v}_2$, $\bm{v}_3$, $\parallel\bm{v}_1\parallel = \parallel\bm{v}_2\parallel = \parallel\bm{v}_3\parallel = 1$ exist in the all source data $\bm{x}_s$ and the linear combinations of the hidden features $k_1\bm{v}_1 + k_2\bm{v}_2 + k_3\bm{v}_3, \{k_i \geq 0 |i \in \{ 1, 2, 3\}$ can be extracted from the feature generator $g$ through supervised learning. The activation function of classifier are set as ReLU and thus $c(g(\bm{x})) \geq 0$. The prediction $\hat{y}=1$ if $c(g(\bm{x})) > 0$, else $\hat{y}=0$.
Because the source domain and the target domain are different, only a part of target data (ratio = $\alpha$, $\alpha < 1$) contains all three hidden features $\bm{v}_1$, $\bm{v}_2$, and $\bm{v}_3$. Assume $\frac{1-\alpha}{3}$ target data only contains one hidden feature $\bm{v}_1$, $\bm{v}_2$, or $\bm{v}_3$, respectively. After training $c_1$ and $c_2$ to maximize their discrepancy and training $g$ to minimize the classifier discrepancy in the target domain, there is a probability that $g$ can only learn a subset of features which means $\exists k \in \{k_1, k_2, k_3\}, k = 0$.
\end{theorem}

\begin{proof}
Because $c_1$ and $c_2$ are trained to classify source data correctly:
\begin{equation}
\label{eq:supervised_constraints}
\begin{split}
&c_1(g(\bm{x}))=c_1(k_1\bm{v}_1 + k_2\bm{v}_2 + k_3\bm{v}_3) > 0\\
&c_2(g(\bm{x}))=c_2(k_1\bm{v}_1 + k_2\bm{v}_2 + k_3\bm{v}_3) > 0
\end{split}
\end{equation}

Then the classifier discrepancy in the target domain is:
\begin{equation}
\label{eq:classifer_discrepancy}
\begin{split}
\epsilon_{T_g}(c_1, c_2) &= \bm{E}_{g(\bm{x}) \sim T_g}\big[|c_1(g(\bm{x})) - c_2(g(\bm{x}))|\big]\\
&= \alpha |c_1(k_1\bm{v}_1 + k_2\bm{v}_2 + k_3\bm{v}_3) - c_2(k_1\bm{v}_1 + k_2\bm{v}_2 + k_3\bm{v}_3)| \\
&+ \frac{1-\alpha}{3}\big[|c_1(k_1\bm{v}_1)-c_2(k_1\bm{v}_1)| + |c_1(k_2\bm{v}_2)-c_2(k_2\bm{v}_2)| \\
&+ |c_1(k_3\bm{v}_3)-c_2(k_3\bm{v}_3)|\big]\\
&= 0 + \frac{1-\alpha}{3}\big[|c_1(k_1\bm{v}_1)-c_2(k_1\bm{v}_1)| + |c_1(k_2\bm{v}_2)-c_2(k_2\bm{v}_2)| \\
&+ |c_1(k_3\bm{v}_3)-c_2(k_3\bm{v}_3)|\big] \\
&\leq \frac{1-\alpha}{3} \big[|c_1(k_1\bm{v}_1)| +|c_2(k_1\bm{v}_1)| + |c_1(k_2\bm{v}_2)|+|c_2(k_2\bm{v}_2)| \\
&+ |c_1(k_3\bm{v}_3)|+|c_2(k_3\bm{v}_3)|\big] \big]
\end{split}
\end{equation}
where $g(\bm{x}_t) = k_i \bm{v}_i, i\in\{1, 2, 3\}$ if $\bm{x}_t$ only contains one feature $\bm{v}_i$. The maximum value is achieved if and only if $\{c_1(k_i\bm{v}_i) c_2(k_i\bm{v}_i) \leq 0 | i\in\{1, 2, 3\}\}$.

To maximize the $\epsilon_{T_g}(c_1, c_2)$, we can get a result that:
\begin{equation}
\begin{split}
\exists c_j, c_j(v_i) = 
\begin{cases}
  >0 & \text{if $i=l, l \in \{1, 2, 3\}$} \\
  0  & \text{otherwise}
\end{cases}
\end{split}
\end{equation}
where $l$ is a random index.

If a classifier can recognize two features, without loss of generality, 
we can assume $c_1(\bm{v}_1) = 1, c_1(\bm{v}_2)= c_1(\bm{v}_3)= 0 $ while $c_2(\bm{v}_1) = 0, c_2(\bm{v}_2)= c_2(\bm{v}_3) = 1$.

Because $c_1$ and $c_2$ are linear:
\begin{equation}
\label{eq:k_constraint}
\begin{split}
c_1(g(\bm{x})) &= c_1(k_1\bm{v}_1 + k_2\bm{v}_2 + k_3\bm{v}_3) = k_1c_1(\bm{v}_1) + k_2c_1(\bm{v}_2 + k_3c_1(\bm{v}_3) \\ &= k_1 > 0 \\
c_2(g(\bm{x})) &= c_2(k_1\bm{v}_1 + k_2\bm{v}_2 + k_3\bm{v}_3) = k_1c_2(\bm{v}_1) + k_2c_2(\bm{v}_2 + k_3c_2(\bm{v}_3) \\ &= k_2 + k_3 > 0\\
\end{split}
\end{equation}

The last step of MCD is to train the feature generator to minimize the classifier discrepancy: 
\begin{equation}
\begin{split}
\min_{k_1, k_2, k_3} \epsilon_{T_g}(c_1, c_2) &= \min_{k_1, k_2, k_3} \frac{1-\alpha}{3}\big[|k_1c_1(\bm{v}_1)-k_1c_2(\bm{v}_1)| \\ &+ |k_2c_1(\bm{v}_2)-k_2c_2(\bm{v}_2)| + |k_3c_1(\bm{v}_3)-k_3c_2(\bm{v}_3)|\big] \\
&=\min_{k_1, k_2, k_3}\frac{1-\alpha}{3}\big[k_1 + k_2 + k_3\big]
\end{split}
\end{equation}

Because $k_1 + k_2 + k_3 \geq k_1 + k_2$ and $k_1 + k_2 + k_3 \geq k_1 + k_3$ as $\{k_i \geq 0 | i\in\{1, 2, 3\}\}$, the condition of minimizing the classifier discrepancy is that $\exists i \in \{2, 3\}, k_i = 0$. 

Since $\exists i \in \{2, 3\}, k_i = 0$, $c(g(\bm{x}_t)) = 0$ if $\bm{x}_t)$ only contains feature $\bm{v}_i$. Then the target accuracy is $1 - \frac{1-\alpha}{3}$.   

According to the symmetry, we can find that the training result of MCD for the single model is that g can only learn a subset of features and $\exists k \in \{k_1, k_2, k_3\}, k = 0$ if a classifier can recognize two features. In this case, the target accuracy is $1 - \frac{1-\alpha}{3}$.

If both two classifiers can only recognize one feature, without loss of generality, 
we can assume $c_1(\bm{v}_1) = 1, c_1(\bm{v}_2)= c_1(\bm{v}_3)= 0 $ while $c_2(\bm{v}_1) = c_2(\bm{v}_3) = 0, c_2(\bm{v}_2) = 1$.

In this case, minimizing the classifier discrepancy equals: 
\begin{equation}
\begin{split}
\min_{k_1, k_2, k_3} \epsilon_{T_g}(c_1, c_2) &= \min_{k_1, k_2, k_3} \frac{1-\alpha}{3}\big[|k_1c_1(\bm{v}_1)-k_1c_2(\bm{v}_1)| \\ &+ |k_2c_1(\bm{v}_2)-k_2c_2(\bm{v}_2)| + |k_3c_1(\bm{v}_3)-k_3c_2(\bm{v}_3)|\big] \\
&=\min_{k_1, k_2, k_3}\frac{1-\alpha}{3}\big[k_1 + k_2\big]
\end{split}
\end{equation}
where the $k_3$ does not affect the result and thus $k_3$ can be either 1 or 0, and the probability of each case is 50\%.

Then $c(g(\bm{x}_t)) = 0$ if $\bm{x}_t$ only contains feature $\bm{v}_3$. The target accuracy will be $1 - \frac{1-\alpha}{3}$.

According to the above two situations, there is a high probability ($p > 50\%$) that the feature generator only learns a subset of features. Moreover, the target accuracy in this theoretical experiment is $1 - \frac{1-\alpha}{3}$ for the single model.
\end{proof}

Since the single learners may only learn a subset of features, ensemble multiple learners may learn all features. In the \autoref{theorem:ensemble_learners}, we will demonstrate that maximizing the feature diversity allow feature generators to learn all features.

\begin{theorem}
\label{theorem:ensemble_learners}
(Ensemble learners) Let $g_1$, $g_2$, and $g_3$ be three feature generators. Each feature generator is able to learn the linear combination of three orthogonal hidden features $k_{i1}\bm{v}_1 + k_{i2}\bm{v}_2 + k_{i3}\bm{v}_3$, $\parallel\bm{v}_1\parallel = \parallel\bm{v}_2\parallel = \parallel\bm{v}_3\parallel = 1$. After maximizing the diversity of features: 
\begin{equation}
\begin{split}
L_{d,g} = \bm{E}_{\bm{x}}\Big[\sum_{i=1}^{3}\lvert g_i(\bm{x}) - \bm{E}_{i\in[1,3]} g_i(\bm{x})\lvert \Big].
\end{split}
\end{equation}
the ensemble of three learners will learn all hidden features:
\begin{equation}
\begin{split}
\exists i, k_{ij} > 0, i\in\{1, 2, 3\} \text{ for every } j \in \{1, 2, 3\}.
\end{split}
\end{equation}
\end{theorem}

\begin{proof}
According to \autoref{theorem:single_learner}, the discrepancy of features:
\begin{equation}
\begin{split}
L_{d,g} &= r_s \bm{E}_{\bm{x}_s}\Big[\sum_{i=1}^{3}\lvert g_i(\bm{x}_s) - \bm{E}_{i\in[1,3]} g_i(\bm{x}_s)\lvert \Big] \\ 
&+ r_t \bm{E}_{\bm{x}_t}\Big[\sum_{i=1}^{3}\lvert g_i(\bm{x}_t) - \bm{E}_{i\in[1,3]} g_i(\bm{x}_t)\lvert \Big]\\
&= (r_s + r_t\alpha)\sum_{i=1}^{3}\lvert \sum_{j=1}^{3} (k_{ij}-\bm{E}_{i\in[1,3]}k_{ij})\bm{v}_j\lvert \\
&+ r_t\frac{1-\alpha}{3}\sum_{i=1}^{3} \sum_{j=1}^{3} \lvert(k_{ij}-\bm{E}_{i\in[1,3]}k_{ij})\bm{v}_j\lvert
\end{split}
\end{equation}
where $r_s$ and $r_t$ are the ratio of the sample number of source data and the target data to the total sample number, respectively.

For $L_1$ (Manhattan) distance:
\begin{equation}
\begin{split}
\lvert \sum_{i=1}^{3} k_i \bm{v}_i\lvert = \sum_{i=1}^{3} \big\lvert k_i \parallel \bm{v}_i\parallel \big\lvert=
\sum_{i=1}^{3} \lvert k_i\lvert.
\end{split}
\end{equation}

Hence:
\begin{equation}
\begin{split}
L_{d,g} &= (r_s + r_t\alpha)\sum_{i=1}^{3} \sum_{j=1}^{3} \lvert k_{ij}-\bm{E}_{i\in[1,3]}k_{ij}\lvert \\
&+ r_t\frac{1-\alpha}{3}\sum_{i=1}^{3} \sum_{j=1}^{3} \lvert k_{ij}-\bm{E}_{i\in[1,3]}k_{ij}\lvert
\\&= (r_s + r_t\frac{1+2\alpha}{3})\sum_{i=1}^{3} \sum_{j=1}^{3} \lvert k_{ij}-\bm{E}_{i\in[1,3]}k_{ij}\lvert
\end{split}
\end{equation}

Without loss of generality, we can assume $k_{1j} \leq k_{2j} \leq k_{3j}$. Then:
\begin{equation}
\begin{split}
\sum_{i=1}^{3} \lvert k_{ij}-\bm{E}_{i\in[1,3]}k_{ij}\lvert &= \bm{E}_{i\in[1,3]}k_{ij} - k_{1j} + k_{3j} - \bm{E}_{i\in[1,3]}k_{ij} + \lvert k_{2j} - \bm{E}_{i\in[1,3]}k_{ij} \lvert\\
&= k_{3j} - k_{1j} + \lvert k_{2j} - \frac{k_{1j} + k_{2j} + k_{3j}}{3}\lvert\\
&=k_{3j} - k_{1j} + \lvert \frac{2}{3}(k_{2j} - \frac{k_{1j}+k_{3j}}{2})\lvert \\
&\leq \frac{5}{3}(k_{3j} - k_{1j})
\end{split}
\end{equation}
where the maximum value is achieved when $k_{2j} = k_{3j}$ or $k_{2j} = k_{1j}$.

Therefore, maximizing the feature discrepancy will cause max$\{k_{ij} | i \in\{1, 2, 3\}\} > 0$ and min$\{k_{ij} | i \in\{1, 2, 3\}\} = 0$. For every $j \in \{1, 2, 3\}, \exists i = \argmax_{i \in\{1, 2, 3\}} k_{ij}, k_{ij} > 0$. Hence, all three hidden features will be learned. 

According to \eqref{eq:k_constraint}, each feature generator should learn at least two hidden features. Therefore, for each $\bm{x}_t$ that only contains one feature $\bm{v}_l$, $\exists i = \argmin_{i \in\{1, 2, 3\}} k_{ij}|j\neq l$, $g_i(\bm{x}_s) = k_{il}\bm{v}_l + k_{im} \bm{v}_m, m \neq j$. Because the classifiers are trained to maximize the classifier discrepancy, which is shown in \eqref{eq:classifer_discrepancy}, $c_{i1}(k_{im} \bm{v}_m) c_{i2}(k_{im} \bm{v}_m) \leq 0$. Therefore, $\exists n \in {1, 2}, c_{in}{k_{im} \bm{v}_m} = 0$. According to \eqref{eq:supervised_constraints}, the source data should be classified accurately: 
\begin{equation}
\begin{split}
&c_{in}(g_i(\bm{x}_s)) > 0 \\
&c_{in}(k_{il}\bm{v}_l + k_{im}\bm{v}_m) > 0 \\
&k_{il}c_{in}\bm{v}_l + 0 > 0 \\
&c_{in}\bm{v}_l > 0
\end{split}
\end{equation}

In conclusion, for each $\bm{x}_t$ which contains feature $\bm{v}_l$, $\exists i \in \{1, 2, 3\} \text{ and } n\in\{1, 2\}$, $c_{in}(g_i(\bm{x}_t)) = k_{il}c_{in}(\bm{v}_l) > 0$. Therefore, the target accuracy for the ensemble learners in this theoretical experiment is 100\%.
\end{proof}

\begin{theorem}
\label{theorem:knowledge_distillation}
(Knowledge distillation) Let $\{g_i, c_{in}| i \in \{ 1, 2, 3\}, n\in \{1, 2\}\}$, be the teacher learners mentioned \autoref{theorem:ensemble_learners}. A student learner with only one feature generator $g_s$ and one classifier $c_s$ can also classify all target accurately $g_s(\bm{x}) =  k_{s1}\bm{v}_1 + k_{s2}\bm{v}_2 + k_{s3} \bm{v}_3, \{k_{sj} > 0 | j\in \{ 1, 2, 3\}\}$ and $\{c_s(\bm{v}_j) > 0| j \in \{ 1, 2, 3\}\}$, by learning the soft labels of the teach network.
\end{theorem}

\begin{proof}
For $\bm{x}_t$ that only contains feature $\bm{v}_j$, the student learner is trained to by the soft label of the teacher learner. As proved in \autoref{theorem:ensemble_learners}, $\exists i \in \{1, 2, 3\} and n \in \{1, 2\}$, $c_{in}(g_i{\bm{x}_t}) > 0$. Then the soft label, which is average of ensemble learners, $\bar{y}(\bm{x}_t) > 0$. After minimizing the cross entropy between $\bar{y}(\bm{x}_t)$ and $c_s(g_s(\bm{x}_t))$, $c_s(g_s(\bm{x}_t)) > 0$.

Because $\bm{x}_t$ only contains feature $\bm{v}_j$,
\begin{equation}
\begin{split}
&c_s(g_s(\bm{x}_t)) = c_s(k_{sj}\bm{v}_j) = k_{sj}
c_s(\bm{v}_j) > 0 \\
&\rightarrow k_{sj} > 0 \text{ and } c_s(\bm{v}_j) > 0, j \in \{1, 2, 3\}.
\end{split}
\end{equation}
\end{proof}

\bibliographystyle{elsarticle-num}
\bibliography{references}